\documentclass[11pt]{article}
\pdfoutput=1
\usepackage[square,sort,comma,numbers]{natbib}
\usepackage{amsmath,amsthm,amssymb,natbib,url}
\usepackage[margin=1in]{geometry}
\usepackage{caption}
\usepackage{subfig}
\usepackage[pdftex]{graphicx}
\usepackage{amsfonts}
\usepackage{mathtools}
\usepackage{amssymb}

\usepackage{comment}
\usepackage{nicefrac}
\usepackage{thmtools, thm-restate}
\newtheorem{theorem}{Theorem}
\newtheorem{proposition}{Proposition}[section]
\newtheorem{Conjecture}{Conjecture}[section]

\newtheorem{lemma}[theorem]{Lemma}
\newtheorem{definition}{Definition}
\usepackage{hyperref}       

\newcommand{\E}{\textbf{E}}
\newcommand{\Var}{\textbf{Var}}
\newcommand{\eps}{\epsilon}

\usepackage{authblk}
\begin{document}
\title{Implicit regularization for deep neural networks driven by an Ornstein-Uhlenbeck like process}
\author[1]{Guy Blanc\thanks{gblanc@cs.stanford.edu}}
\author[1]{Neha Gupta\thanks{nehagupta@cs.stanford.edu}}
\author[1]{Gregory Valiant\thanks{gvaliant@cs.stanford.edu}}
\author[2]{Paul Valiant\thanks{pvaliant@gmail.com}}
\affil[1]{Stanford University}
\affil[2]{Purdue University; Institute for Advanced Study}
\maketitle

\begin{abstract}
We consider networks, trained via stochastic gradient descent to minimize $\ell_2$ loss, with the training labels perturbed by independent noise at each iteration. We characterize the behavior of the training dynamics near any parameter vector that achieves zero training error, in terms of an implicit regularization term corresponding to the sum over the data points, of the squared $\ell_2$ norm of the gradient of the model with respect to the parameter vector, evaluated at each data point.  This holds for networks of any connectivity, width, depth, and choice of activation function.  We interpret this implicit regularization term for three simple settings: matrix sensing, two layer ReLU networks trained on one-dimensional data, and two layer networks with sigmoid activations trained on a single datapoint.  For these settings, we show why this new and general implicit regularization effect drives the networks towards ``simple'' models. 
\end{abstract}
\section{Introduction}
This work is motivated by the grand challenge of explaining---in a rigorous way---why deep learning performs as well as it does. Despite the explosion of interest in deep learning, driven by many practical successes across numerous domains, there are many basic mysteries regarding why it works so well.  Why do networks with orders of magnitude more parameters than the dataset size, trained via stochastic gradient descent (SGD), often yield trained networks with small generalization error, despite the fact that such networks and training procedures are capable of fitting even randomly labeled training points~\cite{zhang2016understanding}?  Why do deeper networks tend to generalize better, as opposed to worse, as one might expect given their increased expressivity?  Why does the test performance of deep networks often continue to improve after their training loss plateaus or reaches zero?




%

In this paper, we introduce a framework that sheds light on the above questions.  Our analysis focuses on deep networks, trained via SGD, but where the gradient updates are computed with respect to noisy training labels.  Specifically, for a stochastic gradient descent update for training data point $x$ and corresponding label $y$, the gradient is computed for the point $(x,y+Z)$ for some zero-mean, bounded random variable $Z$, chosen independently at each step of SGD.   We analyze this specific form of SGD with independent label noise because such training dynamics seem to reliably produce ``simple'' models, independent of network initialization, even when trained on a small number of data points.  This is not true for SGD without label noise, which has perhaps hindered attempts to rigorously formalize the sense in which training dynamics leads to ``simple'' models.  In Section~\ref{sec:futureDirections}, however, we discuss the possibility that a variant of our analysis might apply to SGD without label noise, provided the training set is sufficiently large and complex that the randomness of SGD mimics the effects of the explicit label noise that we consider.

Our main result, summarized below, characterizes the zero-training-error attractive fixed points of the dynamics of SGD with label noise and $\ell_2$ loss, in terms of the local optima of an implicit regularization term.

\begin{theorem}[informal]\label{thm:main}
Given training data $(x_1,y_1),\ldots,(x_n,y_n)$, consider a model $h(x,\theta)$ with bounded derivatives up to $3^{\textrm{rd}}$ order, and consider the dynamics of SGD, with independent bounded label noise of constant variance, and $\ell_2$ loss function $\frac{1}{n}\sum_{i=1}^n (h(x_i,\theta)-y_i)^2$.
A parameter vector $\theta^*$ with 0 training error will be an attractive fixed point of the dynamics if and only if $\theta^*$ is a local minimizer of the ``implicit regularizer''
\begin{equation} reg(\theta) = \frac{1}{n}\sum_{i=1}^n \|\nabla_{\theta} h(x_i,\theta)\|_2^2, \label{eq:implicit} \end{equation} when restricted to the manifold of 0 training error.
\end{theorem}

While the exact dynamics are hard to rigorously describe, the results here are all consistent with the following \emph{nonrigorous} caricature: SGD with label noise proceeds as though it is optimizing not the loss function, but rather the loss function plus the implicit regularizer times the product of the learning rate ($\eta$) and the standard deviation of the label noise. Thus, for small learning rate, SGD with label noise proceeds in an initial phase where the training loss is optimized to 0, followed by a second phase where the implicit regularizer is optimized within the manifold of training error 0.

\subsection{Implications and Interpretations}

We illustrate the implications of our general characterization in three basic settings for which the implicit regularization term can be analyzed: matrix sensing as in~\cite{li2017algorithmic}, two-layer ReLU networks trained on one-dimensional data, and 2-layer networks with logistic or tanh activations trained on a single labeled datapoint.  In all three cases, empirically, training via SGD with label noise yields ``simple'' models, where training without label noise results in models that are not simple and that depend on the initialization.  The intuitive explanation for this second point is clear: there is a large space of models that result in zero training error. Once optimization nears its first 0-training-error hypothesis, optimization halts, and the resulting model depends significantly on the network initialization.  In the following three examples, we argue that the combination of zero training error and being at a local optima of the implicit regularizer reduces the set of models to \emph{only} ``simple'' ones. 

\paragraph{Matrix sensing: Convergence to ground truth from any initialization.}

\cite{li2017algorithmic} consider the problem of ``matrix sensing": given a set of linear ``measurements" of an unknown matrix $X^*$---namely, inner products with randomly chosen matrices $A_i$, can one find the lowest-rank matrix $X$ consistent with the data? They found, quite surprisingly, that gradient descent \emph{when initialized to an overcomplete orthogonal matrix of small Frobenius norm} implicitly regularizes by, essentially, the rank of $X$, so that the lowest-rank $X$ consistent with the data will be recovered provided the optimization is not allowed to run for too many steps. 

Intriguingly, our implicit regularizer allows SGD with label noise to reproduce this behavior for \emph{arbitrary} initialization, and further, without eventually overfitting the data. We outline the main ingredients here and illustrate with empirical results.  As in \cite{li2017algorithmic}, we take our objective function to be minimizing the squared distance between each label $y_i$ and the (Frobenius) inner product between the data matrix $A_i$ and the symmetrized hypothesis $X=UU^\top$: \[\min_U \sum_i \left(y_i-\langle A_i,UU^\top\rangle\right)^2\] 

In our notation, this corresponds to having a hypothesis---as a function of the parameters $U$ and the data $A_i$---of $h(A_i,U)=\langle A_i,UU^\top\rangle$. Thus, taking data drawn from the i.i.d. normal distribution, the implicit regularizer, Equation~\ref{eq:implicit}, is seen to be \begin{equation}reg(U)=\mathop{\E}_{A:A_{j,k}\leftarrow\mathcal{N}(0,1)}[\,||(A+A^\top)U)||_F^2\,\label{eq:matrix-regularizer}]\end{equation}

Because $A$ has mean 0 and covariance equal to the identity matrix, for $d\times d$ matrices this expectation is calculated to be $2(d+1)||U||_F^2$. Further, expressed in terms of the overall matrix $X=UU^\top$, the squared Frobenius norm of $U$ equals the \emph{nuclear norm} of $X$, $||U||_F^2=||X||_*$, where the nuclear norm may be alternatively defined as the convex envelope of the rank function on matrices of bounded norm. See e.g.~\cite{candes2009exact,recht2010guaranteed} for discussion of conditions under which minimizing the nuclear norm under affine constraints guarantees finding the minimum rank solution to the affine system. In our setting, each data point $(A_i,y_i)$ induces an affine constraint on $UU^\top$ that must be satisfied for the training error to be 0, and thus the implicit regularizer will tend to find the minimum of Equation~\ref{eq:matrix-regularizer} subject to the constraints, and hence the minimum rank solution subject to the data, as desired.  We note that the above intuitive analysis is only in expectation over the $A_i$'s, and we omit an analysis of the concentration. However, empirical results, in Figure~\ref{fig:matrixSensing}, illustrate the success of this regularizing force, in the natural regime where $n=5\cdot rank \cdot dimension.$

\begin{figure}
    \centering
    \captionsetup{font=small}
    \includegraphics[width=0.6\textwidth]{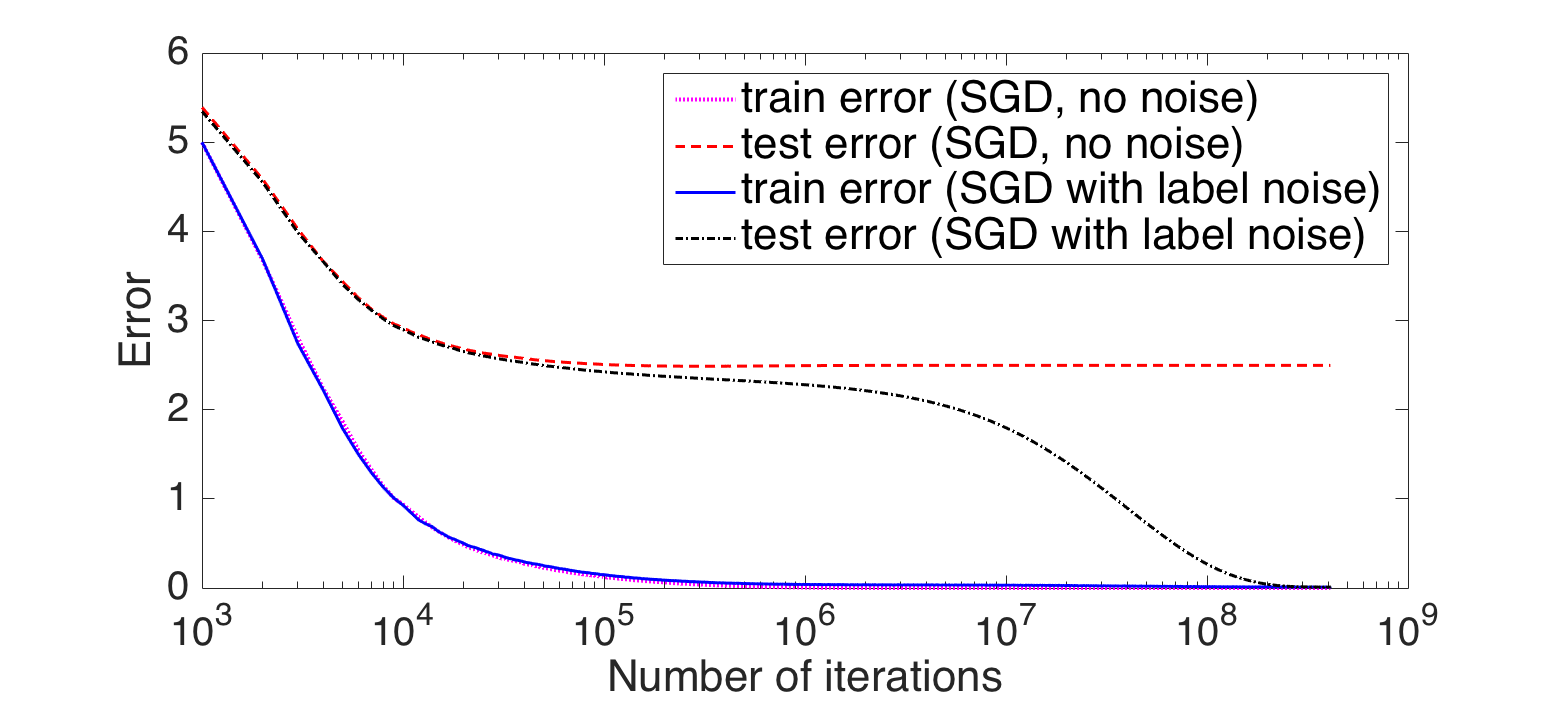}
\caption{Illustration of the implicit regularization of SGD with label noise in the \emph{matrix sensing} setting (see~\cite{li2017algorithmic}).  Here, we are trying to recover a rank $r$ $d\times d$ matrix $X^*=U^* U^{*\top}$ from $n=5dr$ linear measurements $A_1,\langle A_1,X^*\rangle,\ldots,A_n,\langle A_n,X^*\rangle$, via SGD both with and without label noise, with $r=5$ and $d=100$, and entries of $A_i$ chosen i.i.d. from the standard Gaussian.  Plots depict the test and training error for training with and without i.i.d. $N(0,0.1)$ label noise, initializing $U_0=I_d$.  (Similar results hold when $U_0$ is chosen with i.i.d. Gaussian entries.) For both training dynamics, the training error quickly converges to zero.  The test error without label noise plateaus with large error, whereas the test error with label noise converges to zero, at a longer timescale, inversely proportional to the \emph{square} of the learning rate, which is consistent with the theory.}
    \label{fig:matrixSensing}
\end{figure}

\paragraph{2-Layer ReLU networks, 1-d data: Convergence to piecewise linear interpolations.}  Consider a 2-layer ReLU network of arbitrary width, trained on a set of 1-dimensional real-valued datapoints, $(x_1,y_1),\ldots,(x_n,y_n)$.  Such models are not differentiable everywhere, and hence Theorem~\ref{thm:main} does not directly apply.  Nevertheless, we show that, if one treats the derivative at the ``kink'' in the ReLU function as being 0, then local optima of the implicit regularization term correspond to ``simple'' functions that have the minimum number of convexity changes necessary to fit the training points.  The proof of the following theorem is given in Appendix~\ref{ap:caseAnalysis}.

\begin{theorem}\label{thm:informal_1}
\sloppy Consider a dataset of 1-dimensional data, corresponding to $(x_1,y_1),\ldots,(x_n,y_n)$ with $x_i < x_{i+1}$.  Let $\theta$ denote the parameters of a 2-layer ReLU network  (i.e. with two layers of trainable weights) and where there is an additional constant and linear unit leading to the output.  Let $\theta$ correspond to a function with zero training error.  If the function, restricted to the interval $(x_1,x_n)$, has more than the minimum number of changes of convexity necessary to fit the data, then there exists an infinitesimal perturbation to $\theta$ that 1) will preserve the function value at all training points, and 2) will decrease the implicit regularization term of Equation~\ref{eq:implicit}, provided we interpret the derivative of a ReLU at its kink to be 0 (rather than undefined).
\end{theorem}

The above theorem, together with the general characterization of attractive fixed points of the dynamics of training with label noise (Theorem~\ref{thm:main}), suggest that we should expect this noisy SGD training to lead to ``simple'' interpolations of the datapoints; specifically, for any three co-linear points, the interpolation should be linear. This behavior is supported by the experiments depicted in Figure~\ref{fig:1d}, which also illustrates the fact that training without label noise produces models that are not simple, and that vary significantly depending on the network initialization. 

\begin{figure}[h]
  \includegraphics[width=1\textwidth]{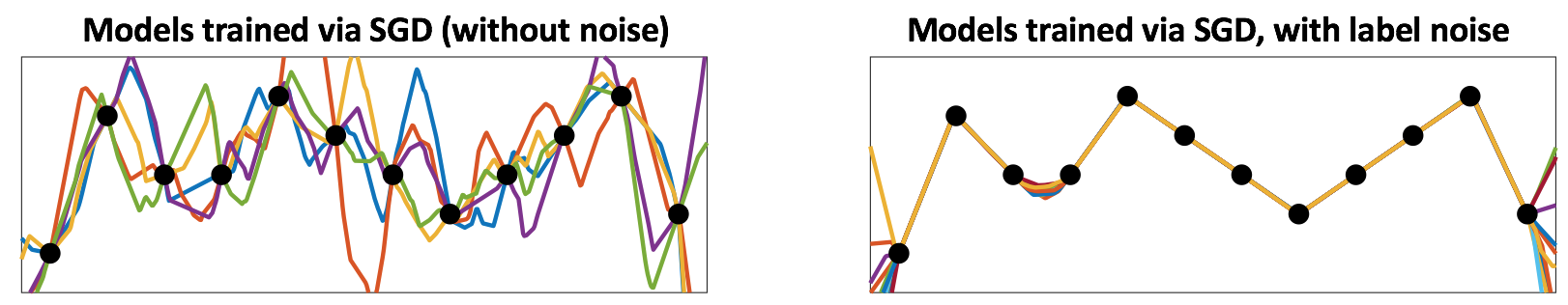}
  \vspace*{-5pt}
\caption{\small{Both plots depict 2-layer ReLU networks, randomly initialized and trained on the set of 12 points depicted.  The left plot shows the final models resulting from training via SGD, for five random initializations.  In all cases, the training error is 0, and the models have converged.  The right plot shows the models resulting from training via SGD \emph{with independent label noise}, for 10 random initializations. Theorem~\ref{thm:informal_1} explains this behavior as a consequence of our general characterization of the implicit regularization effect that occurs when training via SGD with label noise, given in Theorem~\ref{thm:main}.  Interestingly, this implicit regularization does \emph{not}  occur (either in theory or in practice) for ReLU networks with only a single layer of trainable weights.}}\vspace*{-5pt}
  \label{fig:1d}
\end{figure}

\begin{figure}[h]
  \centering
  \includegraphics[width=0.4\textwidth]{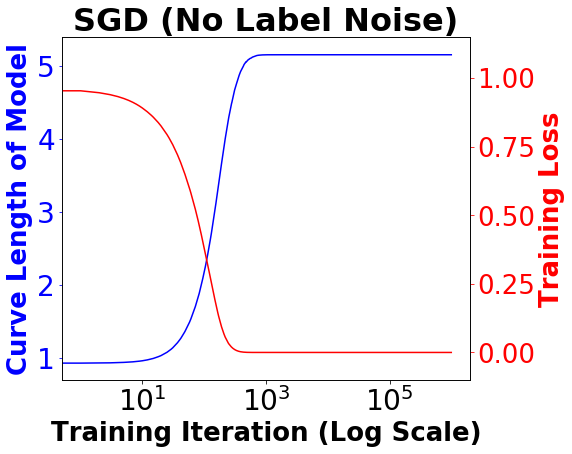}\quad \includegraphics[width=0.4\textwidth]{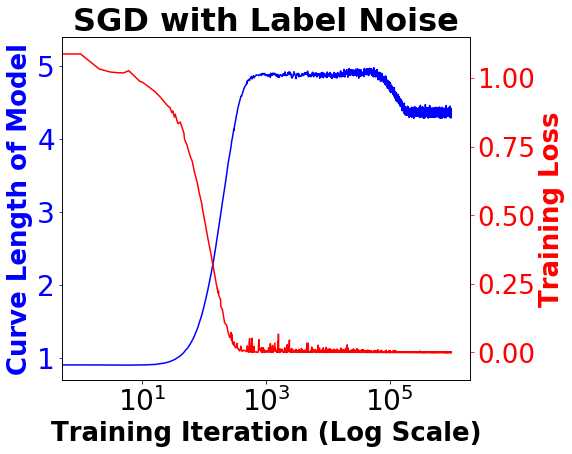}
\caption{\small{Plots depicting the training loss (red) and length of the curve corresponding to the trained model (blue) as a function of the number of iterations of training for 2-layer ReLU trained on one-dimensional labeled data.  The left plot corresponds to SGD without the addition of label noise, and converges to a trained model with curve length $\approx 5.2$.  The right plot depicts the training dynamics of SGD with independent label noise,  illustrating that training first finds a model with close to zero training error, and then---at a much longer timescale---moves within the zero training error manifold to a ``simpler'' model with significantly smaller curve length of $\approx 4.3$.  Our analysis of the implicit regularization of these dynamics explains why SGD with label noise favors simpler solutions, as well as why this ``simplification'' occurs at a longer timescale than the initial loss minimization.}\label{fig:curveLength}}
\end{figure}

\paragraph{2-Layer sigmoid networks, trained on one datapoint: Convergence to sparse models.}

 Finally, we consider the implicit regularizer in the case of a two layer network (with arbitrary width) with logistic or hyperbolic tangent activations, when trained on a dataset that consists of a single labeled $d$-dimensional point.  The proof of this result is given in Appendix~\ref{ap:n=1}.

\begin{theorem}
\label{thm:n=1}
    Consider a dataset consisting of a single $d$-dimensional labeled point,  $(x,y)$. Let $\theta = (\{c_i\}, \{w_i\})$ denote the parameters of a 2-layer network with arbitrary width, representing the function
        $f_\theta(x) = \sum_{i=1}^n c_i \sigma (w_i^t x),$
    where the activation function $\sigma$ is either  tanh or the logistic activation. If $\theta$ corresponds to a model with zero training error for which the implicit regularizer of Equation~\ref{eq:implicit} is at a local minimum in the zero training error manifold, then there exists $\alpha_1, \alpha_2$ and $\beta_1, \beta_2$ such that for each hidden unit $i$, either $c_i = \alpha_1$ and $\sigma(w_i^t x) = \beta_1$, or $c_i = \alpha_2$ and $\sigma(w_i^t x) = \beta_2$, or $c_i=\sigma(w_i^tx) = 0$.  In the case of tanh activations, $\alpha_1=-\alpha_2$ and $\beta_1=-\beta_2$.
\end{theorem}
The above theorem captures the sense that, despite having arbitrarily many hidden units, when trained on an extremely simple dataset consisting of a single training point, the stable parameters under the training dynamics with label noise correspond to simple models that do not leverage the full expressive power of the class of networks of the given size.

\subsection{Future Directions}\label{sec:futureDirections}
There are a number of tantalizing directions for future research, building off the results of this work.  One natural aim is to better understand what types of stochasticity in the training dynamics lead to similar implicit regularization.  In our work, we consider SGD with independently perturbed labels.  These training dynamics are equivalent to standard SGD, performed over a dataset where each original datapoint $(x,y)$ has two ``copies'', corresponding to $(x,y+\delta)$ and $(x,y-\delta)$.  In this setting with two perturbed copies of each data point, the implicit regularization can be viewed as arising from the stochasticity of the choice of datapoint in SGD, together with the fact that no model can perfectly fit the data (since each $x$-value has two, distinct, $y$ values).  Motivated by this view, one natural direction would be to  rigorously codify the sense in which implicit regularization arises from performing SGD (without any additional noise) over ``difficult-to-fit'' data.  Figure~\ref{fig:1d} illustrates the importance of having difficult-to-fit data, in the sense that if the training loss can be driven close to zero too quickly, then training converges before the model has a chance to forget its initialization or ``simplify''.  One hope would be to show that, on any dataset for which the magnitude of each SGD update remains large for a sufficiently large number of iterations, a similar characterization to the implicit regularization we describe, applies.  

In a different direction, it seems worthwhile characterizing the implications of Theorem~\ref{thm:main} beyond the matrix sensing setting, or the 1-dimensional data, 2-layer ReLU activation setting, or the single datapoint tanh and sigmoid settings we consider.  For example, even in the setting of 1-dimensional data, it seems plausible that the characterization of Theorem~\ref{thm:main} can yield a result analogous to Theorem~\ref{thm:informal_1} for ReLU networks of any depth greater than 2 (and any width), as opposed to just the 2-layer networks we consider (and empirically, the analogous claim seems to hold).  For 2-layer networks with tanh or sigmoid activations, it seems likely that our proof could be generalized to argue that: any non-repellent set of parameters for a dataset of at most $k$ points has the property that there are only $O(k)$ classes of activations.

The question of generalizing the characterization of non-repellent parameters from the 1-dimensional data setting of Theorem~\ref{thm:informal_1} to higher dimensional data seems particularly curious.  In such a higher dimensional setting, it is not even clear what the right notion of a ``simple'' function should be.  Specifically, the characterization that the trained model has as few changes in convexity as is required to fit the data does not seem to generalize in the most natural way beyond one dimension.


Finally, it may also be fruitful to convert an understanding of how ``implicit regularization drives generalization'' into the development of improved algorithms.  Figure~\ref{fig:curveLength} and our results suggest that the implicit regulization which drives generalization occurs at a significantly longer time scale than the minimization of the objective function: the training dynamics rapidly approach the zero training error manifold, and then very slowly traverse this manifold to find a simpler model (with better generalization).  It seems natural to try to accelerate this second phase, for example, by making the regularization explicit.  More speculatively, if we understand why certain implicit (or explicit) regularizations yield models with good generalization, perhaps we can directly leverage a geometric understanding of the properties of such models to directly construct functions that interpolate the training set while having those desirable properties, entirely circumventing SGD and deep learning: for example, Theorem~\ref{thm:informal_1} and Figure~\ref{fig:1d} show a setting where SGD becomes essentially nearest-neighbor linear interpolation of the input data (where the distance metric can be viewed as a kernel analogous to the ``neural tangent kernel" of~\cite{jacot2018neural} ), a simple model that can be both justified and computed without reference to SGD.

\subsection{Related Work}

There has been much recent interest in characterizing which aspects of deep learning are associated with robust performance. We largely restrict our discussion to those works with provable results, though the flavor of those  results is rather different in each case.

An influential paper providing a rigorous example of how gradient descent can be effective despite more trainable parameters than training examples is the work of Li et al. on matrix sensing~\cite{li2017algorithmic}. In their setting (which is closely related to 2-layer neural networks with 1 out of the 2 layers of weights being trainable), they optimize the coefficients of an $n\times n$ matrix, subject to training data that is consistent with a low-rank matrix. What they show is that, for sufficiently small initial data, the coefficients essentially stay within the space of (approximately) low-rank matrices. And thus, while the number of trainable parameters is large ($n\times n$), gradient descent can effectively only access a space of dimension $k\times n$, where $k\ll n$ is the rank of the training data. 
This paper marks a key example of provable ``algorithmic regularization'': the gradient descent algorithm leads to more felicitous optima than are typical, given the parameterization of the model. A few high-level differences between these results and ours include: 1) their results show that the high number of parameters in their setting is essentially an illusion, behind which their model behaves essentially like a low-parameter model, while evolution in our model is a high-dimensional phenomenon; 2) their model is closely related to a neural network with one layer of trainable weights, while we cover much deeper networks, revealing and relying on a type of regularization that cannot occur with only one trainable layer.

As in the above work of Li et al., ~\cite{maennel2018gradient} also proceeds by showing that, when initialized to a parameter vector of small norm, the training dynamics of ``simple'' data converge to a simple hypothesis. They empirically observe that the final function learned by a 2-layer ReLU network on 1-dimensional data, with parameters initialized to have small norm, is a piecewise linear interpolation. For the special case when the $n$ datapoints lie in a line, they prove that the resulting trained functions would have at most $2n+1$ changes in the derivative. 

Several recent papers have shown generalization bounds for neural networks by first describing how different settings lead to an implicit or explicit maximization of the \emph{margin} separating correct predictions from mispredictions. These papers are in a rather different setting from our current work, where data is typically labeled by discrete categories, and the neural network is trained to rate the correct category highly for each training example, while rating all incorrect categories lower by a margin that should be as large as possible. The paper 
by \cite{soudry2018implicit} showed that, under any of several conditions, when the categories are linearly separable, gradient descent will converge to the max-margin classifier. More generally, \cite{wei2018margin} 
showed that optimizing the cross-entropy loss is extremely similar to optimizing the maximum margin, in that, after adding an additional weak regularization term, the global optimum of their loss function provably maximizes the margin. This line of work both leverages and expands the many recent results providing generalization bounds in terms of the margin. 
We caution, however, that the margin is still essentially a loss function on the training data, and so this alone cannot defend against over-parameterization and the often related problems of overfitting. (Our regularizer, by contrast, depends on a \emph{derivative} of the hypothesis, and thus unlike the margin, can discriminate between parameter vectors expressing identical functions on the training data.)


There are also quite different efforts to establish provable generalization, for example~\cite{hardt2015train,kuzborskij2017data}, which argue that if networks are trained for few epochs, then the final model is ``stable'' in that it does not depend significantly on any single data point, and hence it generalizes. Such analyses seem unlikely to extend to the realistic regimes in which networks are trained for large numbers of iterations over the training set. There is also the very recent work tightening this connection between stable algorithms and generalization~\cite{DBLP:journals/corr/abs-1902-10710}.     In a different vein, recent work~\cite{brutzkus2017sgd} establishes generalizability under strong (separability) assumptions on the data for overcomplete networks; this analysis, however, only trains one of the layers of weights (while keeping the other fixed to a carefully crafted initialization).

There has been a long line of work, since the late 1980's, studying the dynamics of neural network training in the presence of different types of noise (see, e.g.~\cite{sietsma1988neural,hanson1990stochastic,clay1992fault,murray1994enhanced,an1996effects,rifai2011adding}).  This line of work has considered many types of noise, including adding noise to the inputs, adding noise to the labels (outputs), adding noise to the gradient updates (``Langevin'' noise), and computing gradients based on perturbed parameters. Most closely related to our work is the paper of \cite{an1996effects}, which explicitly analyzes label noise, but did not analyze it in enough detail to notice the subtle 2nd-order regularization effect we study here, and thus also did not consider its consequences.

There have also been several efforts to rigorously analyze the apparent ability of adding noise in the training to avoid bad local optima.  For example, in~\cite{zhang2017hitting}, they consider Langevin noise---noise added to the gradient updates themselves---and show that the addition of this noise (provably) results in the model escaping from local optima of the empirical risk that do not correspond to local optima of the population-risk.  In a slightly different direction, there is also a significant effort to understand the type of noise induced by the stochasticity of SGD itself.  This includes the recent work~\cite{chaudhari2017stochastic} which empirically observes a peculiar \emph{non-stationary} behavior induced by SGD, and~\cite{zhu2018regularization} which describes how this stochasticity allows the model to tend towards more ``flat'' local minima.

\section{Formal statement of general characterization}

Our general result, Theorem~\ref{thm:main}, applies to any network structure---any width, any depth, any set of (smooth) activation functions.  The characterization establishes a simple condition for whether the training dynamics of SGD with label noise, trained under the $\ell_2$ loss, will drive the parameters away from a given zero training error solution $\theta$.  Our characterization is in terms of an implicit regularization term, proportional to the sum over the data points, of the squared $\ell_2$ norm of the gradient of the model with respect to the parameter vector, evaluated at each data point. Specifically, letting $h(x_i,\theta)$ denote the prediction at point $x_i$ corresponding to parameters $\theta$, the implicit regularizer is defined as \begin{equation}\label{eq:regularizer-intro}reg(\theta) = \frac{1}{n}\sum_{i=1}^n \|\nabla_{\theta} h(x_i,\theta)\|_2^2.\end{equation}  We show that for a zero training error set of parameters, $\theta_0$, if there is a data point $x_i$ and a direction (within the subspace that, to first order, preserves zero training error) where the regularizer has a nonzero gradient, then for any sufficiently small learning rate $\eta,$ if the network is initialized near $\theta_0$ (or passes near  $\theta_0$ during the training dynamics), then with probability $1-exp(-1/poly(\eta)),$ the dynamics will drive the parameters at least distance $D_{\eta}$ from $\theta_0$ after some time $T_{\eta},$ and the value of the implicit regularization term will decrease by at least $\Theta(poly(\eta)).$  On the other hand, if $\theta_0$ has zero gradient of the regularizer in these directions, then with probability $1-exp(-1/poly(\eta)),$ when initialized near $\theta_0$, the network will stay within distance $d_{\eta} = o(D_{\eta})$ up through time $T_{\eta}.$  This characterization corresponds to saying that the training dynamics will be expected to stay in the vicinity of a zero training error point, $\theta_0$, only if $\theta_0$ has zero gradient of the implicit regularizer within the zero training error manifold about $\theta_0$; for the particular time window $T_\eta$, this characterization is strengthened to ``if and only if''.   

To quantify the above characterization, we begin by formally defining the sense in which training dynamics are ``repelled'' from points, $\theta,$ which are not local optima of the implicit regularizer within the manifold of zero training error, and defining the sense in which the dynamics are not repelled in the case where the implicit regularizer has zero gradient in these directions.

\begin{definition}\label{def:repellent}
Let $\theta(t)$ denote the set of parameters of a learning model, trained via $t$ steps of SGD under squared $\ell_2$ loss with independent label noise of unit variance. We say that $\theta^*$, is a ``strongly-repellent'' point, if there is a constant $c>0$ such that for any sufficiently small learning rate, $\eta>0,$ for a network initialized to $\theta(0)$ satisfying $\|\theta(0)-\theta^*\| \le \eta^{0.5},$ then with probability at least $1-exp(-1/poly(\eta))$,  for $t = \eta^{-1.6}:$ \begin{itemize}
    \item $\|\theta(t)-\theta^*\| \ge c\eta^{0.4},$ namely the training dynamics lead away from $\theta^*.$
    \item $reg(\theta(0))-reg(\theta(t)) = c \eta^{0.4} + o(\eta^{0.4}),$ namely, the value of the implicit regularization term decreases significantly.
\end{itemize}
\end{definition}

\begin{definition}\label{def:nonrepellent}
  Given the setup above, we say that $\theta^*$, is a ``non-repellent'' point, if, for any sufficiently small learning rate, $\eta>0,$ for a network initialized to $\theta(0)$ satisfying $\|\theta(0)-\theta^*\| \le \eta^{0.5},$ then with probability at least $1-exp(-1/poly(\eta))$,  for any $t\le \eta^{-1.6},$ it holds that $\|\theta(t)-\theta^*\| \le \eta^{0.44}.$
\end{definition}

The following theorem quantifies the sense in which the implicit regularizer characterizes the dynamics of training, in the vicinity of parameters with zero training error.

\paragraph{Theorem~\ref{thm:main} } \emph{
Consider the dynamics of the parameters, $\theta,$ of a deep network, trained via SGD to minimize $\ell_2$ loss, with independent bounded label noise of unit variance.  Let parameters $\theta^*$  correspond to a model $f(\theta^*,x)$ with zero training error, namely $f(\theta^*,x_i)=y_i$ for all $i=1,\ldots,n$. If the implicit regularizer has zero gradient in the span of directions where $f(\theta^*,x_i)$ has zero gradient, for all $i$, then $\theta^*$ is ``non-repellent'' in the sense of Definition~\ref{def:nonrepellent} (meaning the dynamics will remain near $\theta^*$ with high probability for a sufficiently long time).  Otherwise, if the implicit regularizer has non-zero gradient in the directions spanned by the zero error manifold about $\theta^*$, then $\theta^*$ is ``strongly-repellent'' in the sense of Definition~\ref{def:repellent} (implying that with high probability, the dynamics will lead away from $\theta^*$ and the value of the implicit regularizer will decrease significantly).} \medskip

\section{Intuition of the implicit regularizer, via an Ornstein-Uhlenbeck like analysis}\label{sec:prelims}

The intuition for the implicit regularizer arises from viewing the SGD with label noise updates as an Ornstein-Uhlenbeck like process.  To explain this intuition, we begin by defining the notation and setup that will be used throughout the proof of Theorem~\ref{thm:main}, given in Section~\ref{sec:mainproof}.

\subsection{Preliminaries and Notation}
We consider training a model under stochastic gradient descent, with a quadratic loss function. Explicitly, we fit a parameter vector $\theta$ given training data consisting of pairs $(x_i,y_i)$ where $x_i$ is the $i^\text{th}$ input and $y_i\in\mathbb{R}$ is the corresponding label; a hypothesis function $h(x_i,\theta)$ describes our hypothesis at the $i^\text{th}$ training point. The resulting objective function, under $\ell_2$ loss, is \begin{equation}\label{eq:objective}\sum_i (h(x_i,\theta)-y_i)^2\end{equation} For convenience, we define the error on the $i^\text{th}$ training point to be $e_i(x_i,\theta)=h(x_i,\theta)-y_i$.

We consider stochastic gradient descent on the objective function expressed by Equation~\ref{eq:objective}, with training rate $\eta$, yielding the following update rule, evaluated on a randomly chosen data point $i$:
\begin{equation}\label{eq:gradient-def}
    \theta\leftarrow\theta-\eta\nabla_\theta(e_i(x_i,\theta)^2)
\end{equation}

Our analysis will examine a power series expansion of this SGD update rule with respect to $\theta$, centered around some point of interest, $\theta^*$.  Without loss of generality, and to simplify notation, we will assume $\theta^*=0$ and hence the power series expansions we consider will be centered at the origin. For notational convenience, we use $h_i$ to denote $h(x_i,0)$ and $e_i$ to denote $e(x_i,0)$. To denote derivatives along coordinate directions, we use superscript letters, separated by commas for multiple derivatives: $h_i^j$ denotes the derivative of $h_i$ with respect to changing the $j^\text{th}$ parameter of $\theta$, and $h_i^{j,k}$ represents the analogous 2nd derivative along coordinates $j$ and $k$. (All derivatives in this paper are with respect to $\theta$, the second argument of $h$, since the input data, $\{x_i\}$ never changes.) As a final notational convenience for derivatives, we represent a directional derivative in the direction of vector $v$ with a superscript $v$, so thus $h_i^v=\sum_j v_j h_i^j$, where $v_j$ denotes the $j$th coordinate of $v$; analogously, $h_i^{v,v,j}$ is a 3rd derivative along directions $v$, $v$, and coordinate $j$, defined to equal $\sum_{k,\ell} v_k v_\ell h_i^{j,k,\ell}$.   In our proof of Theorem~\ref{thm:main}, we will only ever be considering directional derivatives in the direction of parameter vector $\theta.$

Our proof of Theorem~\ref{thm:main} will rely on an expansion of the SGD update rule (Equation~\ref{eq:gradient-def}) expanded to 3rd order about the origin. Explicitly, the $j^\text{th}$ coordinate of $\theta_j$ updates according to this equation by $\eta$ times the derivative in the $j^\text{th}$ direction of $e_i(x_i,\theta)^2$. The $k^\text{th}$ order term in the power series expansion of this expression will additionally have a $k^\text{th}$ order directional derivative in the direction $\theta$, and a factor of $\frac{1}{k!}$. Thus the $k^\text{th}$ order term will have one $j$ derivative and $k$ $\theta$ derivatives distributed across two copies of $e_i$; since $e_i(x_i,\theta)=h(x_i,\theta)-y_i$ and $y_i$ has no $\theta$ dependence, any derivatives of $e_i$ will show up as a corresponding derivative of $h_i$. Combining these observations yields the 3rd order expansion of the gradient descent update rule:

\begin{equation}\label{eq:theta_update-intro}
\theta_j\leftarrow \theta_j-2\eta e_i h_i^j - 2\eta (h^\theta_i h^j_i + e_i h^{j,\theta}_i)-\eta(h^j_i h^{\theta,\theta}_i + 2 h^{j,\theta}_i  h^\theta_i + e_i h^{j,\theta,\theta}_i) +O(\eta \theta^3), 
\end{equation}

\noindent where the final big-O term bounds all terms of 4th order and higher.  Throughout, we consider the asymptotics in terms of only the learning rate, $\eta<1$, and hence regard $\theta$, the number and dimension of the datapoints, the size of the network, and all derivatives of $h$ at the origin as being bounded by $O_{\eta}(1)$.  We are concerned with the setting where the label error, $e_i$, has an i.i.d. random component, and assume that this error is also bounded by $O(1)$.  Additionally, since we are restricting our attention to the neighborhood of a point with zero training error, we have that for each $i$, the expectation of $e_i$ is 0.

\subsection{Diagonalizing the exponential decay term}\label{sec:diag}
The 2nd term after $\theta_j$ on the right hand side of the update rule in Equation~\ref{eq:theta_update-intro} is $-2\eta h_i^\theta h_i^j = -2\eta \sum_k \theta_k h_i^k h_i^j$. Ignoring the $-2\eta$ multiplier, this expression equals the vector product of the $\theta$ vector with the $j^\text{th}$ column of the (symmetric) positive semidefinite matrix whose $(j,k)$ or $(k,j)$ entry equals $h_i^j h_i^k$. The expectation of this term, over a random choice of $i$ and the randomness of the label noise, can be expressed as the positive semidefinite matrix $\E_i[h_i^j h_i^k]$, which will show up repeatedly in our analysis. Since this matrix is positive semidefinite, we choose an orthonormal coordinate system whose axes diagonalize this matrix. Namely, without loss of generality, we take $\E_i[h_i^j h_i^k]$ to be a diagonal matrix. We will denote the diagonal entries of this matrix as $\gamma_j=\E_i[h_i^j h_i^j]\geq 0$. Thus, this term of the update rule for $\theta_j$ reduces to $-2\eta\gamma_j\theta_j$ in expectation, and hence this terms corresponds to an exponential decay towards 0 with time constant $1/(2\eta\gamma_j)$.  And for directions with $\gamma_j=0,$ there is no decay. 

Combined with the 1st term after $\theta_j$ on the right hand side of the update rule in Equation~\ref{eq:theta_update-intro}, namely $-2\eta e_i h_i^j$, whose main effect when $e_i$ has expectation near 0 is to add noise to the updates, we have what is essentially an Ornstein-Uhlenbeck process; the 2nd term, analyzed in the previous paragraph, plays the role of mean-reversion. However, because of the additional terms in the update rule, we cannot simply apply standard results, but must be rather more careful with our bounds. However the (multi-dimensional) Ornstein-Uhlenbeck process can provide valuable intuition for the evolution of $\theta$.

\subsection{Intuition behind the implicit regularizer}

Recall that we defined the implicit regularizer of Equation~\ref{eq:regularizer-intro} to be the square of the length of the gradient of the hypothesis with respect to the parameter vector, summed over the training data.  Hence, in the above notation, it is proportional to: \begin{equation}\label{eq:regularizer}\sum_k \E_i[h_i^k h_i^k]\end{equation} The claim is that stochastic gradient descent with label noise will act to minimize this quantity once the optimization has reached the training error 0 regime.  The mechanism that induces this implicit regularization is subtle, and apparently novel. As discussed in Section~\ref{sec:diag}, the combination of the first 2 terms of the $\theta_j$ update in Equation~\ref{eq:theta_update-intro} acts similarly to a multidimensional Ornstein-Uhlenbeck process, where the noise added by the first term is countered by the exponential decay of the second term, converging to a Gaussian distribution of fixed radius. The singular values $\gamma_j$ (defined in Section~\ref{sec:diag}) control this process in dimension $j$, where---ignoring the remaining update terms for the sake of intuition---the Ornstein-Uhlenbeck process will converge to a Gaussian of radius $\Theta(\sqrt{\eta})$ in each dimension for which $\gamma_j>0$.  Crucially, this limiting Gaussian is isotropic!  The variance in direction $j$ depends only on the variance of the label noise and does not depend on $\gamma_j$, and the different dimensions become uncorrelated.  The convergence time, for the dynamics in the $j$th direction to converge to a Gaussian, however, is  $\Theta(\frac{1}{\eta\sqrt{\gamma_j}})$, which varies inversely with $\gamma_j$.

Crucially, once sufficient time has passed for our quasi-Ornstein-Uhlenbeck process to appropriately converge, the expectation of the 5th term of the update in Equation~\ref{eq:theta_update-intro}, $\E[-2\eta h_i^{j,\theta}h_i^\theta]=-2\eta \E[\sum_{k,\ell} \theta_k \theta_\ell h_i^{j,k} h_i^\ell]$, takes on a very special form. Assuming for the sake of intuition that convergence occurs as described in the previous paragraph, we expect each dimension of $\theta$ to be uncorrelated, and thus the sum should consist only of those terms where $k=\ell$, in which case $\E[\theta_k^2]$ should converge to a constant (proportional to the amount of label noise) times $\eta$. Namely, the expected value of the 5th term of the Equation~\ref{eq:theta_update-intro} update for $\theta_j$ should be proportional to the average over data points $i$ of $-2\eta^2 \sum_k h_i^{j,k}h_i^k$, and this expression is seen to be exactly $-\eta^2$ times the $j$ derivative of the claimed regularizer of Equation~\ref{eq:regularizer}. In short, subject to the Ornstein-Uhlenbeck intuition, the 5th term of the update for $\theta_j$ behaves, in expectation, as though it is performing gradient descent on the regularizer, though with a training rate an additional $\eta$ times slower than the rate of the overall optimization.

To complete the intuition, note that the rank of the matrix $\E_i[h_i^j h_i^k]$ is at most the number of datapoints, and hence for any over-parameterized network, there will be a large number of directions, $j$, for which $\gamma_j = 0.$  For sufficiently small $\eta$---any value that is significantly smaller than the smallest nonzero $\gamma_j$---the update dynamics will look roughly as follows:  after $\gg 1/\eta$ updates, for any directions $k$ and $\ell$ with $\gamma_k, \gamma_{\ell}>0$, we have $\E[\theta_k^2]\approx \E[\theta_{\ell}^2] = \Theta(\eta),$ and for $k \neq \ell,$ we have $\E[\theta_k \theta_{\ell}] = 0.$  The update term responsible for the regularization, $2h_i^{j,\theta}h_i^{\theta}$ will not have a significant effect for the directions, $j$, for which $\gamma_j>0$, as these directions have significant damping/mean-reversion force and behave roughly as in the Ornstein-Uhlenbeck process, as argued above.  However, for a direction $j$ with $\gamma_j =0,$ there is no restoring force, and the effects of this term will add up, driving $\theta$ consistently in the direction of the implicit regularizer, restricted to the span of dimensions, $j$, for which $\gamma_j=0$. The full proof of Theorem~\ref{thm:main} is stated in Appendix~\ref{main-theorem-proof}.

\bigskip\noindent{\Large\bf Acknowledgements}\medskip

We would like to thank Hongyang Zhang for suggesting the matrix sensing experiment. The contributions of Guy, Neha, and Gregory, were supported by NSF awards AF-1813049 and CCF-1704417, an ONR
Young Investigator Award, and DOE Award DE-SC0019205. Paul Valiant is partially supported by NSF award IIS-1562657.
\bibliographystyle{plain}
\bibliography{refs}

\newpage
\appendix
\section{Proof of Theorem~\ref{thm:main}}\label{sec:mainproof}
\label{main-theorem-proof}
This section contains the proofs of our general characterization of stable neighborhoods of points with zero training error, under the training dynamics of SGD with label noise.

See Section~\ref{sec:prelims} for notation, intuition, and preliminaries. In particular, the evolution of parameters $\theta$ is governed by the updates of Equation~\ref{eq:theta_update-intro}, which is a 3rd-order power series expansion about the origin. We analyze the regime where optimization has already yielded a parameter vector $\theta$ that is close to a parameter vector with 0 training error (in expectation, setting aside the mean-0 label noise). Without loss of generality and for ease of notation, we take this 0-error parameter vector to be located at the origin.

Our first lemma may be viewed as a ``bootstrapping'' lemma, saying that, if the parameter vector $\theta$ has remained loosely bounded in all directions, then it must be rather tightly bounded in those directions with $\gamma_j>0$. Each of the following lemmas applies in the setting where $\theta$ evolves according to stochastic gradient descent with bounded i.i.d. label noise.

\begin{lemma}\label{lemma:1}
Given constant $\eps>0$, and $T>0$, if it is the case that  $|\theta| \le \eta^{1/4 +\eps}$ for all $t \le T$, then for any $j$ s.t. $\gamma_j>0,$ it holds that with probability at least $1- exp(-poly(1/\eta))$ at time $T$, $|\theta_j| \le \|\theta(0)\|\cdot e^{-\Omega(\eta T)} + \eta^{1/2 - \eps},$ where $\theta(0)$ denotes the value of $\theta$ at time $t=0.$
\end{lemma}
\begin{proof}
For convenience, we restate the update formula for $\theta$, given in Equation~\ref{eq:theta_update-intro}, where $i$ is the randomly chosen training data index for the current SGD update:
\begin{equation}\label{eq:theta_update}
\theta_j\leftarrow \theta_j-2\eta e_i h_i^j - 2\eta (h^\theta_i h^j_i + e_i h^{j,\theta}_i)-\eta(h^j_i h^{\theta,\theta}_i + 2 h^{j,\theta}_i  h^\theta_i + e_i h^{j,\theta,\theta}_i) +O(\eta \theta^3). 
\end{equation}
We will reexpress the update of $\theta_j$ as \begin{equation}
\theta_j(t) = (1-2\eta \gamma_j)\theta_j(t-1) + z_{t-1} + w_{t-1}, \label{eq:new_update}
\end{equation}
where $z_{t-1}$ will be a mean zero random variable (conditioned on $\theta(t-1)$), whose magnitude is bounded by $O(\eta),$ and $w_{t-1}$ is an error term that depends deterministically on $\theta(t-1)$.

To this end, consider the expectation of the third term of the update in Equation~\ref{eq:theta_update}: $\E_i [h_i^{\theta}h_i^j] = \sum_k \E_i[\theta_k h_i^k h_i^j] = \theta_j \gamma_j,$ as we chose our basis such that $\E_i[h_i^k h_i^j]$ is either 0 if $k \neq j$ and $\gamma_j$ if $k=j$.  Hence this third term, together with the first term, gives rise to the $(1-2\eta \gamma_j)\theta_j(t-1)$ portion of the update, in addition to a  contribution to $z_{t-1}$ reflecting the deviation of this term from its expectation.  For $|\theta| = O(1),$ this zero mean term will trivially be bounded in magnitude by $O(\eta).$  The remaining contributions in the update to $z_{t-1}$ all consist of a factor of $\eta$ multiplied by some combination of $e_i$, powers of $\theta$, and derivatives of $h$, each of which are constant, yielding an overall bound of $z_{t-1}=O(\eta)$.

Finally, we bound the magnitude of the error term, $w_{t-1} = -\eta(h_i^j h_i^{\theta,\theta}+2h_i^{j,\theta} h_i^{\theta})+O(\eta \theta^3)$.  The first two terms are trivially bounded by $O(\eta |\theta|^2),$ and hence since $|\theta(t-1)| \le \eta^{1/4+\eps},$ we have that $|w_{t-1}| = O(\eta \cdot \eta^{1/2 + 2\eps}) = O(\eta^{3/2+2\eps}).$ 

Given the form of the update of  Equation~\ref{eq:new_update}, we can express $\theta(T)$ as a weighted sum of the $\theta(0)$,  $z_0,\ldots,z_{T-1},$ and $w_0,\ldots,w_{T-1}$. Namely letting $\alpha = (1-2\eta\gamma_j),$ we have
$$ \theta(T) = \alpha^T \theta(0) + \sum_{t=0}^{T-1} \alpha^{T-t-1}(z_t +w_t).$$
We begin by bounding the contribution of the error term, $$\sum_{t=0}^{T-1} \alpha^{T-t-1}(w_t) = O(\frac{1}{1-\alpha}\cdot \max_t |w_t| ) = O(\frac{1}{\eta}\eta^{3/2+2\eps}) = O(\eta^{1/2+2\eps}),$$ where the $\frac{1}{1-\alpha}$ term is due to the geometrically decaying coefficients in the sum.

To bound the contribution of the portion of the sum involving the $z_t$'s, we apply a basic martingale concentration bound.  Specifically, note that by defining $Z_t = \sum_{i=0}^{t-1} \alpha^{T-i-1}z_i,$ we have that $\{Z_t\}$ is a martingale with respect to the sequence $\{\theta(t)\},$ since the expectation of $z_{t}$ is 0, conditioned on $\theta(t)$.  We now apply the Azuma-Hoeffding martingale tail bound that asserts that, provided $|Z_t-Z_{t-1}| \le c_t$, $\Pr[|Z_T| \ge \lambda] \le 2 e^{-\frac{\lambda^2}{2\sum_t c_t^2}}.$   In our setting, $c_t = \alpha^{T-t-1}|z_{t-1}|,$ and hence $\sum_t c_t^2 = O(1/(1-\alpha^2)\max_t|z_t^2|) = O(\frac{1}{\eta}\eta^2)=O(\eta).$  Hence for any $c>0,$ by taking $\lambda = c\eta^{1/2}$ we have that $\Pr[|Z_T| \ge c \eta^{1/2}] \le 2 e^{O(c^2)}.$  By taking $c=1/\eta^{\eps},$ our proof is concluded.
\end{proof}

\subsection{Analysis of concentration of the time average of $\theta_j\theta_k$ in the $\gamma_j>0$ directions}

The following lemma shows that, at time scales $\gg 1/\eta,$ the average value of the empirical covariance, $\theta_k\theta_{\ell},$ concentrates for directions $k,\ell$ satisfying $\gamma_j+\gamma_{\ell}>0.$  The proof of this lemma can be viewed as rigorously establishing the high-level intuition described in Section~\ref{sec:prelims}, that for each directions $k$ with $\gamma_k>0$, the behavior of $\theta_k$ is as one would expect in an Ornstein-Uhlenbeck process with time constant $\Theta(1/\eta).$

\begin{lemma}\label{lemma:var_concentration}
Let $T \ge 1/\eta^{1.25}$ denote some time horizon, and assume that, at all $t < T,$ we have that $|\theta(t)| \le R\triangleq\eta^{0.5-\beta},$ and for every direction $j$ for which $\gamma_j>0,$ we have $|\theta_j(t)| \le R_{\gamma>0}=\eta^{0.5-\eps},$ for some constants $\frac{1}{12}>\beta > \eps > 0.$ Then for any pair of directions $j\neq k$ such that at least one of $\gamma_j$ or $\gamma_k$ is positive, we have that 
$$\Pr\left[ \left|\frac{1}{T}\sum_{t=0}^T \theta_j \theta_k (t)\right| \ge \eta^{1.25-2\eps-1.5\beta} \right] = O(e^{\eta^{-\eps}}).$$  Similarly for any direction $j$ with $\gamma_j>0$, we have that  $$\Pr\left[ \left|\eta \Var[e_i]-\frac{1}{T}\sum_{t=0}^T \theta_j^2(t)\right| \ge \eta^{1.25-0.5\eps-2\beta} \right] = O(e^{ \eta^{-\eps}}).$$
\end{lemma}
\begin{proof}
Given the update described by Equation~\ref{eq:theta_update}, we derive the following update for the evolution of the second moments of $\theta:$
\begin{align*}\theta_j\theta_k&\leftarrow \theta_j\theta_k-\theta_k \eta(2 e_i h_i^j + 2 (h^\theta_i h^j_i + e_i h^{j,\theta}_i)) -  \theta_j \eta(2 e_i h_i^k + 2 (h^\theta_i h^k_i + e_i h^{k,\theta}_i))+ 4\eta^2 e_i^2 h^j_i h^k_i\\
&+O(\eta \theta (\eta+\theta^2) ).\end{align*}
As in the proof of Lemma~\ref{lemma:1}, we will reexpress this as the sum of three terms: a mean-reversion term, a term with zero expectation (conditioned on the previous value of $\theta$), and an error term.  To this end we analyze each of the above terms.  Each term that has an $e_i$ but not an $e_i^2$ will have expectation 0, and the magnitude of these terms is trivially bounded by $O(\eta |\theta|)$. 

The first nontrivial term is $2\theta_k \eta h_i^{\theta}h_i^j.$  Splitting this into a mean zero portion, and its expectation, we see that 
$\E_i[\eta \theta_k h^\theta_i h^j_i]= \eta \theta_k \E_i[h^\theta_i h^j_i] =\eta \theta_k \E_i[\sum_\ell \theta_\ell h^\ell_i h^j_i].$  Since $\E_i[h^\ell_i h^j_i]$ is 0 unless $\ell=j$, we simplify the above expression to $\eta \theta_k \E_i[\theta_j h^j_i h^j_i]=\eta \theta_k \theta_j \gamma_j.$  Hence this term contributes $-2\eta \theta_k \theta_j \gamma_j$ to the ``mean reversion'' term, and $|2\theta_k \eta h_i^{\theta}h_i^j| = O(\eta |\theta|^2)$ to the bound on the zero mean term.  An analogous argument holds for the symmetric term, $2\theta_j \eta h_i^{\theta}h_i^k,$ which together account for the full mean reversion portion of the update: $\theta_j \theta_k \leftarrow (1 - 2 \eta(\gamma_j+\gamma_k))\theta_j \theta_k+ \ldots.$

Other than the zero expectation terms and the final big ``O'' term, the only remaining term in the update is the $4 \eta^2 e_i^2 h_i^j h_i^k$ term.  Since the error $e_i$ is i.i.d. mean 0 label noise,  we have that the expectation of this term is $E_i[\eta^2 e_i^2 h^j_i h^k_i]= \eta^2 Var[e_i] \gamma_j$, if $j=k$, and 0 if $j \neq k.$  The magnitude of this term is trivially bounded by $O(\eta^2).$

Summarizing, we have the following expression for the update of the variance in the case that $j \neq k$:
\begin{equation}
\theta_j\theta_k(t) = (1-2\eta(\gamma_j+\gamma_k))\theta_j\theta_k(t-1) + z_{t-1} + w_{t-1}, \label{eq:new_var_update}
\end{equation}
and in the case that $j=k,$ we have the following update:
\begin{equation}
\theta_j^2(t) = (1-4 \eta \gamma_j)\theta_j^2(t-1) + 4 \eta^2 \gamma_j \Var[e_i] + z_{t-1} + w_{t-1}, \label{eq:new_var_update_j}
\end{equation}
where the stochastic term $z_{t-1}$ given $\theta(t-1)$, has expectation 0 and magnitude bounded by $|z_{t-1}| =O(\eta |\theta(t-1)|+\eta^2)$, and the deterministic term $w_{t-1}$ has magnitude bounded by $|w_{t-1}| = O(\eta |\theta(t-1)|^3 + \eta^2 |\theta(t-1)|).$

\medskip

We now turn to showing the concentration in the average value of these covariance terms.  The argument will leverage the martingale concentration of the Doob martingale corresponding to this time average, as the values of $\theta$ are revealed.  A naive application, however, will not suffice, as we will not be able to bound the martingale differences sufficiently tightly.  We get around this obstacle by considering the martingale corresponding to revealing entire batches of $S \gg 1/\eta^{1+\eps}$ updates at once.  Hence each step of the martingale will correspond to $S$ updates of the actual dynamics.  The utility of this is that the mean-reversion of the updates operates on a timescale of roughly $1/\eta$---namely after $O(1/\eta)$ timesteps, the updates have mostly ``forgotten'' the initial value $\theta(0)$.  The martingale differences corresponding to these large batches will be fairly modest, due to this mean reversion, and hence we will be able to successfully apply an Azuma-Hoeffding bound to this more granular martingale.

Given some time horizon $T>S>0$, we consider the Doob martingale $Z_0,Z_1,\ldots,Z_{T/S}$ defined by $$Z_i = \E\left[ \sum_{t=0}^T \theta_j\theta_k(t) | \theta(0),\theta(1),\ldots,\theta(i\cdot S) \right].$$  In words, $Z_i$ is the expected average value of $\theta_j \theta_k$ over the first $T$ steps, conditioned on having already seen $iS$ updates of the dynamics.  To analyze the martingale differences for this Doob martingale, it will be helpful to understand what Equations~\ref{eq:new_var_update} and ~\ref{eq:new_var_update_j} imply about the expectation of $\theta_j\theta_k(t')$, given the value of $\theta(t)$ at some $t<t'$.  Letting $\alpha$ denote the mean reversion strength, namely $\alpha:=2\eta(\gamma_j+\gamma_k)$, or $\alpha := 4 \eta \gamma_j$ in the case that we are considering $\theta_j^2$, we have the following expressions for the expectations respectively: 
$$\E[\theta_j\theta_k(t') | \theta(t)] = \left(\theta_j \theta_k(t)\right)(1-\alpha)^{t'-t} + O\left(\min(t'-t, \frac{1}{\alpha})\cdot (\eta |\theta|^3+\eta^2 |\theta|) \right).$$
$$\E[\theta_j^2(t') | \theta(t)] = \left(\theta_j^2(t)\right)(1-\alpha)^{t'-t} + \left(4 \eta^2 \gamma_j \Var[e_i]\right)\frac{1-(1-\alpha)^{t'-t}}{\alpha} + O\left(\min(t'-t, \frac{1}{\alpha})\cdot (\eta |\theta|^3+\eta^2 |\theta|) \right).$$

For any constant $\eps>0$ and $t' \ge t+1/\eta^{1+\eps},$ and any pair of directions, $j,k$ where $\gamma_j+\gamma_k > 0,$ assuming that $|\theta| \le R$ until time $t'$, we have that the the above two equations simplify to:

\begin{equation}
    \E[\theta_j\theta_k(t') | \theta(t)] = O\left(\frac{1}{\alpha} (\eta R^3+\eta^2 R) \right)
    \label{eq:sim1}
    \end{equation}
    \begin{equation}
\E[\theta_j^2(t') | \theta(t)] = \frac{4\eta^2 \gamma_j \Var[e_i]}{\alpha} + O\left(\frac{1}{\alpha}\left(\eta R^3+\eta^2 R \right)\right). \label{eq:sim2}
\end{equation}

Equipped with the above expressions for the conditional expectations, we now bound the martingale differences of our Doob martingale $\{Z_i\}.$  Revealing the values of $\theta$ at times $t=1+i\cdot S$ to $t=(i+1)\cdot S,$ affects the value of $Z_{i+1}$ in three ways: 1) This pins down the exact contributions of $\theta$ at these timesteps to the sum $\frac{1}{T}\sum \theta_j\theta_k$, namely it fixes $\frac{1}{T}\sum_{\ell = 1+iS}^{(1+i)S} \theta_j \theta_k(\ell)$;  2) This alters the expected contribution of the next batch of $S$ terms, namely $\frac{1}{T}\sum_{\ell = 1+(i+1)S}^{(2+i)S} \theta_j \theta_k(\ell)$; and 3) it alters the expected contribution of the remaining terms, $\frac{1}{T}\sum_{\ell = 1+(i+2)S}^T \theta_j \theta_k(\ell)$.  

We now bound the contribution to the martingale differences of each of these three effects of revealing $\theta(1+iS),\ldots,\theta((1+i)S).$  Assuming that, until time $T$ we have $|\theta_j| \le R_{\gamma>0}$ for any $j$ with $\gamma_j > 0,$ and $|\theta_k| \le R$ for every direction, $k$, we can trivially bound the contribution of 1) and 2) towards the martingale differences by $O(S R_{\gamma>0}R/T)$, and $O(S R_{\gamma>0}^2/T)$, in the respective cases where we are considering $\theta_j \theta_k$ where both $\gamma_j$ and $\gamma_k$ are positive, and the case where exactly one of them is nonzero.  This is because at each of the $2S$ timesteps that cases 1) and 2) are considering, each of the terms in the sum is absolutely bounded by $\frac{R_{\gamma>0}^2}{T}$ and $\frac{R_{\gamma>0} R}{T}$ in the respective cases.   For the third case, we leverage Equations~\ref{eq:sim1} and~\ref{eq:sim2}, which reflect the fact that, conditioning on $\theta((i+1)S)$ has relatively little effect on $\theta(t)$ for $t \ge (i+2)S$.    Namely, the total effect over these at most $T$ terms is at most $O(T \frac{1}{T}\frac{1}{\alpha}(\eta R^3+\eta^2 R))= O(R^3 +\eta R).$

Hence the overall martingale differences for $\{Z_i \}$ are bounded by $$O\left(\frac{S R_{\gamma>0}^2}{T}+ R^3 +\eta R \right), \text{   or   } O\left(\frac{S R_{\gamma>0} R}{T}+ R^3 +\eta R \right),$$ depending on whether we are considering a term $\theta_j\theta_k$ corresponding to $\gamma_j,\gamma_k > 0$, or not (and note that the martingale \emph{difference} does not include a contribution from the variance term in Equation~\ref{eq:sim2} since this term has no dependence on $\theta$).  Hence, as our martingale has $T/S$ updates, by standard martingale concentration, letting $d$ denote a bound on the martingale differences, for any $c>0$, the probability that $\theta_j \theta_k(T)$ deviates from its expectation by more than $O\left(cd\sqrt{T/S}\right)$ decreases inverse exponentially with $c^2$.  In the case of $\theta_j^2$ for a direction $j$ with $\gamma_j > 0,$ we have that the differences $d = O\left(\frac{S R_{\gamma>0}^2}{T}+ R^3 +\eta R \right)$.  Hence for $R = \eta^{0.5 - \beta},$ and $R_{\gamma>0} = \eta^{0.5 - \eps},$ we have $d\sqrt{T/S} = O\left(\eta^{1-2\eps} \sqrt{S/T} + \eta^{1.5-3\beta}\sqrt{T/S}\right).$  Equating the two terms inside the big ``O'' results in choosing $S$ such that $\sqrt{S/T} = \eta^{0.25 - 1.5 \beta + \eps},$ in which case martingale bounds yield $$\Pr\left[|Z_{T/S}-Z_0| \ge \eta^{1.25-2\eps - 1.5 \beta}\right] \le \Pr\left[|Z_{T/S}-Z_0| \ge \eta^{-\eps}\cdot \Omega(d\sqrt{T/S})\right] \le  2e^{O(\eta^{-2\eps})}=O(e^{\eta^{-\epsilon}}).$$

 In the case of $\theta_j \theta_k$ where either $\gamma_j$ or $\gamma_k$ is nonzero, we have that the differences $d = O\left(\frac{S R_{\gamma>0}R}{T}+ R^3 +\eta R \right)$.  Hence for $R = \eta^{0.5 - \beta},$ and $R_{\gamma>0} = \eta^{0.5 - \eps},$ we have $d\sqrt{T/S} = O\left(\eta^{1-\eps-\beta} \sqrt{S/T} + \eta^{1.5-3\beta}\sqrt{T/S}\right).$  Equating these two terms results in choosing $S$ such that $\sqrt{S/T} = \eta^{0.25 - \beta + 0.5\eps},$ in which case $$\Pr\left[|Z_{T/S}-Z_0| \ge \eta^{1.25-0.5\eps - 2 \beta}\right] \le \Pr\left[|Z_{T/S}-Z_0| \ge \eta^{-\eps}\cdot \Omega(d\sqrt{T/S})\right] \le O(e^{\eta^{-\eps}}).$$
 
 To conclude the proof of the lemma, note that Equations~\ref{eq:sim1} implies that, in the case of $\theta_j\theta_k,$ $Z_0=\E[\frac{1}{T} \sum_{t=0}^T \theta_j \theta_k (t)] = O(R^3+\eta R) = o(\eta^{1.25}),$ and hence provided at least one of $\theta_j$ or $\theta_k$ is positive, we have: $$\Pr\left[ \left|\frac{1}{T}\sum_{t=0}^T \theta_j \theta_k (t)\right| \ge \eta^{1.25-2\eps-1.5\beta} \right]\le O(e^{\eta^{-\eps}}).$$  Similarly in the case of  $\theta_j^2,$ Equation~\ref{eq:sim2} implies that $Z_0=\E[\frac{1}{T} \sum_{t=0}^T \theta_j^2(t)] = \eta \Var[e_i] + O(R^3+\eta R) = \eta \Var[e_i]+o(\eta^{1.25}),$ and hence  $$\Pr\left[ \left|\eta \Var[e_i]-\frac{1}{T}\sum_{t=0}^T \theta_j^2(t)\right| \ge \eta^{1.25-0.5\eps-2\beta} \right]\le O(e^{\eta^{-\eps}}).$$
 \end{proof}

\subsection{Proof of Theorem~\ref{thm:main}}
The Proof of Theorem~\ref{thm:main} will follow easily from the following lemma, which characterizes the evolution of $\theta_j$ for directions $j$ for which $\gamma_j=0.$  This evolution crucially leverages the characterization of the average value of $\theta_k \theta_{\ell}$ given in Lemma~\ref{lemma:var_concentration}.  Given this characterization of the evolution of $\theta_j$, Lemma~\ref{lemma:1} shows that the directions $j$, for which $\gamma_j>0$, will stay bounded by $\approx \sqrt{\eta}$, completing the proof.

\begin{lemma}\label{lemma:reg_concentration}
If $\|\theta\|=O(\sqrt{\eta})$ at time 0, then for each direction $j$ with $\gamma_j = 0$,  with probability at least $1-O(exp(-1/poly(\eta))),$ after $T\le \eta^{-1.6}$ updates, we have $$\theta_j(T) = \theta_j(0)- 2 T\eta^{2}\Var[e_i]\left(\sum_{k: \gamma_k>0} \E_i[h_i^{j,k}h_i^{k}]\right) + O(\eta^{0.44}).$$  In the case that $T=\eta^{-1.6},$ this expression becomes $$\theta_j(0)-2 \eta^{0.4}\Var[e_i]\left(\sum_{k: \gamma_k>0} \E_i[h_i^{j,k}h_i^{k}]\right) + O(\eta^{0.44}).$$
 \end{lemma}
 \begin{proof}
%
%
The proof will proceed by induction on time $t$, in steps of size $\eta^{-0.1}$.  Let $\epsilon$ be an arbitrary constant strictly between 0 and $\frac{1}{400}$. Assume that, up to some time $t_0$, for all directions $k$, we have  $|\theta_k(t)| \le \eta^{0.4-\eps}\leq \eta^{.25+\eps}$ for all $t \le t_0$.  Hence, by Lemma~\ref{lemma:1}, for all $t\le t_0,$ for any direction $k$ with $\gamma_k >0$, we have the tighter bound $|\theta_k(t)| \le \eta^{1/2- \eps}$ with all but inverse exponential probability. Consider advancing to some time $t_1 \in [t_0,t_0 + \eta^{-0.1}]$. Since only $\leq\eta^{-.1}$ time steps have passed, $\theta$ cannot have moved far, even using very weak bounds on the $\theta$ update. Explicitly, by assumption, all derivatives are bounded by $O(1)$, for directions $k$ with $\gamma_k>0$ and thus after $\leq \eta^{-0.1}$ additional steps of SGD we have $|\theta_k(t_1)| \le |\theta_k(t_0)| + O(\eta \cdot \eta^{-0.1}) \le 2\eta^{1/2 - \eps},$ with all but inverse exponential probability.  Analogously, by our assumption, we also have that $|\theta_j(t_1)| \le 2 \eta^{0.4-\eps},$ for every direction $j$, including those with $\gamma_j =0$. We now analyze the evolution of $\theta_j$ from time $0$ through time $t_1$, leveraging the above bounds on $|\theta_k|$ across all dimensions, $k$, to bootstrap even tighter bounds. 
 
We consider the update given in Equation~\ref{eq:theta_update}, and again, since we are considering a direction for which $h_i^j = 0$ for all $i$, there is no mean reversion term. Let $r_{k,\ell}:=\E_i[h_i^{j,k}h_i^\ell]$.  Note that for any $\ell$ with $\gamma_{\ell}=0,$ $r_{k,\ell}=0$.  We can reexpress the expectation of the corresponding portion of the update as $$\E_i[h_i^{j,\theta}h_i^{\theta}] = \sum_{k,\ell} r_{k,\ell} \theta_k \theta_{\ell}.$$  Analogously with the martingale analysis in Equations~\ref{eq:new_update},~\ref{eq:new_var_update}, or~\ref{eq:new_var_update_j}, we express the updates of $\theta_j$ as: \begin{equation}\label{eq:update-thm}\theta_j(t)=\theta_j(t-1) - 2 \eta \sum_{k,\ell}\theta_k(t-1)\theta_{\ell}(t-1)r_{k,\ell} + z_{t-1} + w_{t-1},\end{equation}
where $\E[z_{t-1}|\theta(t-1)]=0$ is a mean zero term, defined as
$$z_{t-1} = e_i\eta \left(-2 h_i^{j,\theta(t-1)} -  h_i^{j,\theta(t-1),\theta(t-1)}\right) -2 \eta\left(h_i^{j,\theta(t-1)}h_i^{\theta(t-1)} - \E_i[h_i^{j,\theta(t-1)}h_i^{\theta(t-1)}]\right).$$  Hence $|z_{t-1}| = O(\eta|\theta|)=\eta^{1.4-\eps}.$
The error term satisfies  $|w_{t-1}| = O\left(\eta |\theta(t-1)|^3\right) = O(\eta^{1+3(0.4-\eps)}),$ where the above analysis follows from inspection of the update rule in Equation~\ref{eq:theta_update}, simplifying using the fact that, in our context, $h^j_i=0$ for all $i$.

From our bound on $|z_t|$, the Azuma-Hoeffding martingale concentration bounds give that, $\Pr[|\sum_{t=0}^{t_1-1}z_t| \ge \eta^{1.4-2\eps}\sqrt{t_1}] \le 2e^{-c \eta^{-2\eps}}$.  Additionally, $\sum_{t=0}^{t_1-1} w_t = O(t_1 \eta^{2.2-3\eps}).$  If $t_1\le \eta^{-1.25},$ Lemma~\ref{lemma:var_concentration} does not apply, but we have that $\theta_k(t)\theta_{\ell}(t) = O(\eta^{1/2-\eps+0.4-\eps})$ as long as either $\gamma_k > 0$ or $\gamma_{\ell}>0,$ and hence $\eta \sum_{t=0}^{t_1}r_{k,\ell}\theta_k(t)\theta_{\ell}(t) = O(t_1 \eta^{1.9-2\eps}) = O(\eta^{-1.25+1.9-2\eps})=O(\eta^{1/2}),$ since $r_{k,\ell}=O(1)$.  If $t_1 \leq \eta^{-1.25},$ the martingale concentration also gives a bound of $|\sum_{t=0}^{t_1} z_t| =O(\eta^{1/2})$ with probability $1-O(exp(-1/poly(\eta)))$. Hence if $t_1 \leq \eta^{-1.25},$ then with probability $1-O(exp(-1/poly(\eta))),$ at all times $t\le t_1,$ we have that $|\theta_j| \le \eta^{0.4-\eps}.$ Thus inductively applying this argument, and taking a union bound over these $poly(1/\eta)$ steps, yields that this conclusion holds up through time $t_1=\eta^{-1.25}.$

We now consider the case when $t_1 \in [\eta^{-1.25}, \eta^{-1.6}]$.  In this case, we may apply Lemma~\ref{lemma:var_concentration}, with $\beta = 0.1+\eps$, which guarantees that for a direction $k$ with $\gamma_k >0$, with all but $exp(-1/poly(\eta))$ probability, \begin{equation}\label{eq:two-bounds}\frac{1}{t_1}\sum_{t=0}^{t_1-1}\theta_k^2(t) = \eta \Var[e_i] + O( \eta^{1.05-4\eps}) \text{  and  } \left| \frac{1}{t_1}\sum_{t=0}^{t_1-1}\theta_k(t)\theta_{\ell}(t) \right|= O(\eta^{1.05-4\eps}).\end{equation}  From above, we have that $\sum_{t=0}^{t_1} z_t = O(\eta^{1.4-2\eps}\sqrt{t_1})= O(\eta^{0.6-2\eps})=O(\eta^{1/2}),$ and $\sum_{t=0}^{t_1} w_t = O(t_1 \eta^{2.2-3\eps})=O(\eta^{1/2}).$  From Equation~\ref{eq:update-thm}, we plug in the two bounds from Equation~\ref{eq:two-bounds} multiplied by $\eta$ to conclude that $$\theta_j(t_1) = \theta_j(0)-2\eta^2 t_1 \Var[e_i]\sum_k r_{k,k} + O(t_1 \eta^{2.05-4\eps}) + O(\eta^{1/2}).$$   Note that the ``cross terms,'' $\sum_{k,\ell} r_{k,\ell}$ do not explicitly appear in the previous sum, and instead contribute to the first big ``O'' term, due to our bound on the time average of $\theta_k \theta_{\ell}$ from Lemma~\ref{lemma:var_concentration}.  

Applying the above conclusions inductively (as we did in the first half of the proof for the case $t_1\leq\eta^{-1.25}$) yields that, with all but $exp(-1/poly(\eta))$ probability, $|\theta_j(t)| \le \eta^{0.4-\eps}$ at all times $t \le \eta^{-1.6},$ and at time $T \le \eta^{-1.6},$ we have that $\theta_j(T) = \theta_j(0)-2\eta^2 T \Var[e_i]\sum_k r_{k,k} + O(T \eta^{2.05-4\eps}) + O(\eta^{1/2}) = \theta_j(0)- 2 T\eta^{2}\Var[e_i]\sum_k r_{k,k} + O(\eta^{0.45-4\eps}),$ yielding the lemma, as desired.  
 \end{proof}

\section{Proof of Theorem~\ref{thm:informal_1}
~\label{ap:caseAnalysis}}
%
%

Before proving Theorem~\ref{thm:informal_1}, we formalize the notation that will be used throughout this section.  We consider a network with two layers of trainable weights, with an additional linear and bias unit leading to the output.  The network takes as input a one dimensional datapoint, $x$, and a constant, which we can assume wlog to be 1.  For the $i$th neuron in the middle layer, there are three associated parameters: $a_i$, the weight to input $x$, $b_i$ the weight to the constant input, and $c_i$, the weight from neuron $i$ to the output.  Hence the parameters $\theta=(\{a_i\}, \{b_i\}, \{c_i\},a,b)$ represent the following function: $$f_{\theta}(x) = \sum_{i}c_i\sigma(a_ix + b_i) + ax + b$$ where $\sigma$ is the ReLU non-linearity i.e. $\sigma(x) = max(0,x).$

The implicit regularization term for a dataset $(x_1,y_1),\ldots,(x_n,y_n)$, evaluated at parameters $\theta$, simplifies as follows:
\begin{eqnarray}
R(\theta) & := & \sum_j \|\nabla_{\theta}f_{\theta}(x_j)\|_2^2 \\&=& \sum_j\left( \|\nabla_{\{a_i\}}f_{\theta}(x_j)\|_2^2 + \|\nabla_{\{b_i\}}f_{\theta}(x_j)\|_2^2 + \|\nabla_{\{c_i\}}f_{\theta}(x_j)\|_2^2 + \|\nabla_{a,b}f_{\theta}(x_j)\|_2^2 \right)\nonumber \\
& = & \sum_j  \left(\sum_i (c_i x_j I_{a_ix_j+b_i > 0})^2 + (c_i I_{a_ix_j+b_i>0})^2 + (\sigma(a_ix_j+b_i))^2 \right)+x_j^2+1 \nonumber 
\end{eqnarray}
Defining the contribution of the $i$th ReLU neuron and $j$th datapoint to be $$R_{i,j}(\theta):=(\sigma(a_ix_j + b_i))^2 + c_i^2(1 + x_j^2)I_{a_ix_j + b_i > 0},$$ the regularization expression simplifies to $R(\theta) = \sum_{i,j} R_{i,j}(\theta)+\sum_j 1+x_j^2,$ where the last sum does not depend on $\theta$, thus has no $\theta$ gradient, and thus does not contribute to regularization.

\begin{definition}\label{def:convex kink}
The $i$th ReLU unit $f_i(x) = c_i\sigma(a_ix+b_i)$   has an \emph{intercept} at location $x = -\frac{b_i}{a_i}$, and we say this unit is \emph{convex} if $c_i>0$ and is concave if $c_i<0$.  If $c_i=0,$ then $f_i(x)=0$ and the unit has no effect on the function.
\end{definition}

\begin{proof}[Proof of Theorem~\ref{thm:informal_1}]
The proof will proceed by contradiction, considering a set of parameters, $\theta,$ and set of consecutive datapoints $(x_i,y_i), (x_{i+1},y_{i+1}), (x_{i+2},y_{i+2})$ that violates the claim, and then exhibiting a direction in which $\theta$ could be perturbed that preserves the values of the hypothesis function at the data points, but decreases the implicit regularizer proportionately to the magnitude of the perturbation.

Assume, that the piecewise linear interpolation of $(x_i,y_i), (x_{i+1},y_{i+1}), (x_{i+2},y_{i+2})$ is convex (i.e. concave up).  An analogous argument will apply to the case where it is convex down.  If $f(\theta,x)$ fits the three points, but is not also convex, then it must have a change of convexity, and hence there must be at least two ``kinks'' in the interval $(x_i,x_{i+2}),$ each corresponding to a ReLU unit whose intercept lies in this interval, and with one of the units corresponding to a ``convex'' unit (with $c>0$) and the other a ``concave'' unit (with $c<0$).  We will consider the case where the intercept of the concave unit, $k_1$ is less than the intercept of the convex unit, $k_2$ and the argument in the alternate case is analogous.  There are now three cases to consider: 1) the point $x_{i+1}$ lies between the intercepts, $x_{i+1}\in (k_1,k_2)$; 2) $x_{i+1} = k_1$ or $k_2$; and 3) there is no point in the interval $[k_1,k_2]$.  In each case, we will exhibit a perturbation of the two units in question that simultaneously preserves the function values at all data points $\{x_i\}$, while decreasing the implicit regularizer.  The proof in the first case will trivially also apply to the third case.

We begin with the first case, when $x_{i+1} \in (k_1,k_2).$  For notational convenience, we will henceforth use $x_0$ to denote $x_{i+1}$.  Let $a_1,b_1,c_1$ denote the parameters of the first unit, and $a_2,b_2,c_2$ denote the parameters of the second unit in question.  Figure~\ref{fig:allcases} depicts the setting where $x_{0} \in (k_1,k_2),$ along with the four possible configurations of the units, according to the four possible configurations of the signs of $a_1$ and $a_2$.  In each case, the dotted lines in the figure indicate the direction of perturbation of these two units which 1) preserves the function value at all data points, and 2) decreases the value of the implicit regularizer.  We note that in several of the cases, the bias unit, $b$, and and linear unit, $a$, which are directly connected to the output, must also be adjusted to accomplish this.  We will never perturb the weights, $c_1$, $c_2$, leading to the output neuron.

 \begin{figure}
  \centering
 \includegraphics[width=1\linewidth]{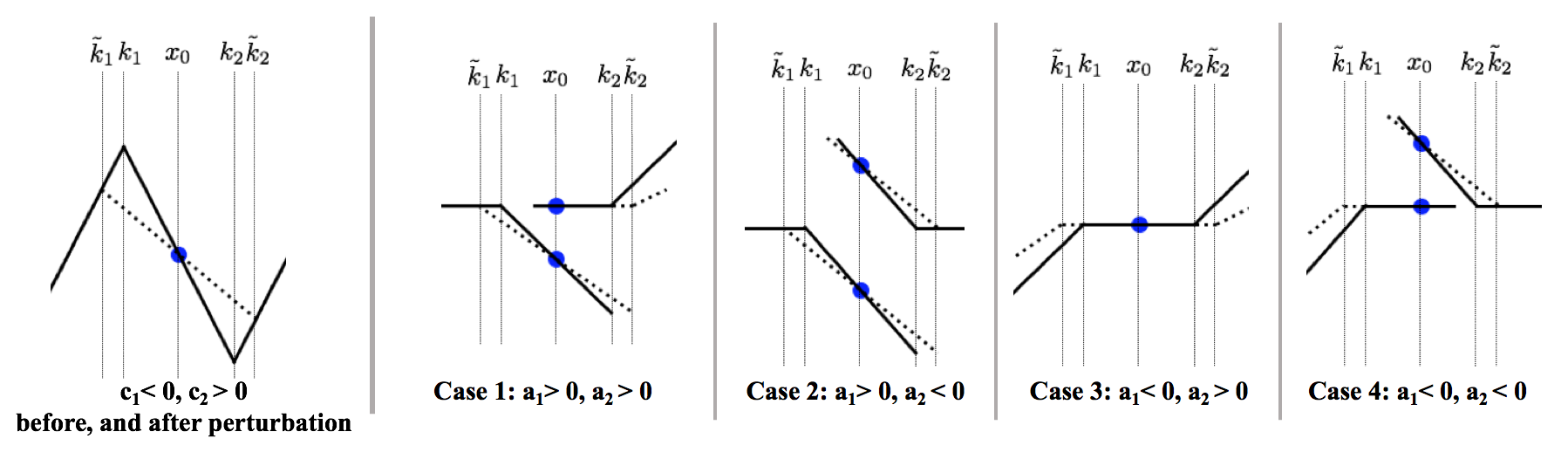}
  \caption{\small{The leftmost figure depicts the case where the middle datapoint lies between the intercepts of the ReLU units with opposing convexities.  The solid line depicts the original function, and the dotted line depicts the function after the perturbation, which preserves the function values at all datapoints and decreases the regularization expression.  The rightmost four plots depict the four possible types of ReLU units that could give rise to the function depicted in the left pane, together with the perturbations that realize the effect depicted in the left pane.  For cases 2 and 3, the linear and bias units must also be adjusted to preserve the function values at the datapoints.\label{fig:allcases}}}
\end{figure}

Let $\tilde{a}_1, \tilde{b}_1$ and $\tilde{a}_2, \tilde{b}_2$ be the parameters of the two perturbed ReLU units and $\tilde{k}_1$ and $\tilde{k}_2$ be the new location of the corresponding intercepts.  The perturbations will be in terms of an arbitrarily small quantity  $\epsilon>0$, and hence we will assume that, for all $j \neq i+1$, $x_{j} \not \in [\tilde{k}_1,\tilde{k}_2].$  Let $\tilde{R}_{1j}, \tilde{R}_{2j}$ denote the contributions to the regularization expression for units 1 and 2 corresponding to the jth datapoint, after the perturbation.

\paragraph{Case 1 (\boldmath $a_1 > 0, c_1 < 0, a_2 > 0, c_2 > 0$): \unboldmath} 
We first give an intuitive argument of how the perturbation is chosen to preserve the function values while decreasing the regularization. As depicted in the second pane of Figure \ref{fig:allcases}, we change the parameters of the first ReLU unit $a_1$ and $b_1$ such that the intercept ${k}_1$ moves towards the left to a position $\tilde{k}_1$ and the slope $a_1$ decreases. The changes in $a_1$ and $b_1$ are chosen such that the value at the point $x_0$ remains the same. The second ReLU unit's parameters are perturbed such that for all datapoints $x_j \geq \tilde{k}_2$, the change in the function values due to the changes in the parameters of the first ReLU unit are balanced by them.  Hence, the function values are preserved for all datapoints.  To see that the regularization decreases by the same order of magnitude as the perturbation, recall that the regularization term for a ReLU unit $i$ and datapoint $j$ is proportional to $(\sigma(a_ix_j + b_i))^2$ if the value of $c_i$ is kept unchanged. From Figure~\ref{fig:allcases}, the value of $(\sigma(a_ix_j + b_i))^2$ for both units remains the same for all datapoints $x_j \leq x_0$ and strictly decreases (proportionately to the magnitude of the perturbation) for all datapoints $x_j \geq \tilde{k}_2$.  This realizes the intuition that the implicit regularizer promotes small activations in the network.

A nearly identical argument applies in the other three cases depicted in Figure~\ref{fig:allcases}, with the slight modification in cases 2 and 3 that we need to perturb the linear and bias units to preserve the function values, and the regularization term is independent of the values of those parameters. 

 Now, we explicitly describe the case analysis mentioned above, and explicitly state the perturbations, and compute the improvement in the regularizer for all four cases, and the cases corresponding to the setting where the data point $x_0$ lies at one of the intercepts, $k_1$ or $k_2$ are analogous.  For clarity, Figure~\ref{fig:bothcases} depicts the function before the perturbation, and after, for both the case when $x_0$ lies between the intercepts $k_1,k_2$, and when $x_0=k_1.$
 \begin{figure}[h]
\centering
\subfloat[When the datapoint is between the kinks.\label{fig:point_bet_kink}]{%
\includegraphics[width=0.3\textwidth]{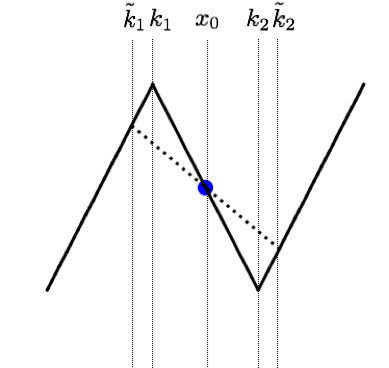}}\hspace{20pt}
\subfloat[When the datapoint is on one of the kinks.\label{fig:point_on_kink}]{%
\includegraphics[width=0.3\textwidth]{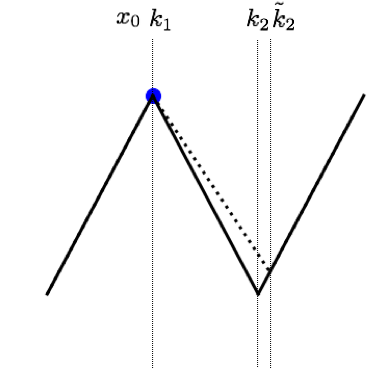}}
\caption{The plots show the  change such that the function values at the datapoints are preserved and the regularization term strictly decreases.}
\label{fig:bothcases}
\end{figure}
%
%
%
%
%
%
We begin by computing the perturbations for each of the four cases depicted in Figure~\ref{fig:allcases}. When the values of linear and bias units $a,b$ are not mentioned, we assume there is no change in them.
 \paragraph{Case 1 (\boldmath $a_1 > 0, c_1 < 0, a_2 > 0, c_2 > 0):$\unboldmath}
    \begin{equation*}
\begin{gathered}
    \tilde{a}_1 = a_1(1 - \epsilon) \qquad \tilde{b}_1 = b_1 + x_0a_1\epsilon\\
\tilde{a}_2= a_2  -\frac{c_1}{c_2}(\tilde{a}_1 - a_1) \qquad \tilde{b}_2= b_2  -\frac{c_1}{c_2}(\tilde{b}_1 - b_1)
\end{gathered}
\end{equation*}
    First, observe that the intercept for ReLU 1 moves to the left since
    $$\tilde{k}_1 - k_1 = -\frac{\tilde{b}_1}{\tilde{a}_1} + \frac{{b}_1}{{a}_1} = -\frac{\epsilon(a_1x_0+b_1)}{a_1(1-\epsilon)} < 0 $$
    The last inequality follows since $0 < \epsilon < 1$ and $a_1 > 0$ and $a_1x_0 + b_1 > 0$ since $x_0 > k_1$ and $a_1k_1 + b_1 = 0$.
Similarly, the intercept for ReLU 2 moves to the right $$\tilde{k}_2 - k_2 = -\frac{\tilde{b}_2}{\tilde{a}_2} + \frac{{b}_2}{{a}_2} = \frac{a_1c_1\epsilon(a_2x_0+b_2)}{a_2(c_2a_2+c_1a_1\epsilon)} > 0 $$

The last inequality follows because $c_1 < 0$, $a_1,a_2 > 0$, $a_2x_0+b_2 < 0$ and $c_2a_2+c_1a_1\epsilon > 0$ for sufficiently small $\epsilon$. Now, we will verify that $f(x_j) = \tilde{f}(x_j) \; \forall \; x_j, j \in [n]$ and the total regularization term $R$  decreases by $\Theta(\eps)$. We will analyze the three cases separately where $x_j \leq \tilde{k}_1$, $x_j = x_0$ and $x_j \geq \tilde{k}_2$. 

     \paragraph{\boldmath $x_j \leq \tilde{k}_1$: \unboldmath}
        Since both the units were not active for $x_j \leq \tilde{k}_1$ and are not active after the change, there is no change in the function value.  Similarly, since the units were not active before the change and did not become active after the change, the regularization term for  $x_j \leq \tilde{k}_1$ does not change.
    \paragraph{\boldmath $x_j = x_0$: \unboldmath}
    First, calculating the value of $\tilde{a}_1x_0 + b_1$, we get that
    \begin{equation}
    \tilde{a}_1x_0 + b_1 = a_1(1 - \epsilon)x_0 + b_1 + a_1\epsilon x_0 = a_1x_0 + b_0 \label{eqn:case12}
    \end{equation}
The function value for $x_0$ does not change since the contribution of the first unit does not change by \eqref{eqn:case12} and the second unit remains off before and after the change. This is by design as we decreased the slope $a_1$ and moved the intercept $k_1$ to the left such that function value at point $x_0$ is preserved.
    \begin{equation*}
    \tilde{f}(x) - f(x) = c_1\sigma(\tilde{a}_1x_0 + \tilde{b}_1) + c_2\sigma(\tilde{a}_2x_0 + \tilde{b}_2) - c_1\sigma({a}_1x_0 + {b}_1) - c_2\sigma({a}_2x_0 + {b}_2) = 0 
    \end{equation*}
  Calculating the change in regularization value with the perturbed parameters, we see there is no change since $\tilde{a_1}x_0 + \tilde{b}_0 = {a_1}x_0 + {b}_0$ by \eqref{eqn:case12} and $c$ does not change. 
     \begin{equation*}
     \tilde{R}_{10} - {R}_{10} = (\sigma(\tilde{a}_1x_0 + \tilde{b}_1))^2 + c_1^2(1 + x_0^2)I_{\tilde{a}_1x_0 + \tilde{b}_1 > 0} - (\sigma(\tilde{a}_1x_0 + \tilde{b}_1))^2 - c_1^2(1 + x_0^2)I_{\tilde{a}_1x_0 + \tilde{b}_1 > 0} = 0
     \end{equation*}
     Since the second unit remains off for $x_0$ before and after the change, the regularization value does not change.
          \begin{equation*}
     \tilde{R}_{20} - R_{20} = (\sigma(\tilde{a}_2x_0 + \tilde{b}_2))^2 + c_2^2(1 + x_0^2)I_{\tilde{a}_2x_0 + \tilde{b}_2 > 0} - (\sigma({a}_2x_0 + {b}_2))^2 - c_2^2(1 + x_0^2)I_{{a}_2x_0 + {b}_2 > 0} = 0 
     \end{equation*}
      Thus, we see that both the function value and the regularization term do not change for $x_0$.
    \paragraph{\boldmath $x_j \geq \tilde{k}_2$: \unboldmath}
    Now for this case, both the units are active before and after the change. So, we need to look at the how the total contribution changes to both the output value and the regularization for both the units. First, calculating $\tilde{a}_1x_j + \tilde{b}_1 - ({a}_1x_j + {b}_1) $, we see that it is strictly negative since $\epsilon > 0, a_1 > 0$ and $x_j \geq \tilde{k}_2 > k_2 > x_0$.
    \begin{equation}
        \tilde{a}_1x_j + \tilde{b}_1 - ({a}_1x_j + {b}_1)  = a_1(1-\epsilon)x_j + b_1 + a_1\epsilon x_0 - {a}_1x_j - {b}_1 = \epsilon a_1(x_0 - x_j)
        \label{eq:case12} < 0
    \end{equation}
    
    Similarly, calculating $\tilde{a}_2x_j + \tilde{b}_2 - ({a}_2x_j + {b}_2)$, we see that it is also strictly negative since $c_1 < 0$ and $c_2 > 0$.
    \begin{equation}
          \tilde{a}_2x_j + \tilde{b}_2 - ({a}_2x_j + {b}_2)  = (\tilde{a}_2-a_2)x_j + \tilde{b}_2 - {b}_2 = -\frac{c_1}{c_2}((\tilde{a}_1-a_1)x_j + \tilde{b}_1 - {b}_1) = -\frac{c_1}{c_2}\epsilon a_1(x_0 - x_j)
           \label{eq:case13}
    \end{equation}  
    This can also be readily seen from the figure \ref{fig:allcases}.
        Now, calculating the change in function value due to the perturbed parameters, we get
        \begin{align*}
    \tilde{f}(x_j) - f(x_j) &= c_1\sigma(\tilde{a}_1x_j + \tilde{b}_1) + c_2\sigma(\tilde{a}_2x_j + \tilde{b}_2) - c_1\sigma({a}_1x_j + {b}_1) - c_2\sigma({a}_2x_j + {b}_2) \\
    &= c_1((\tilde{a}_1 - a_1)x_j + \tilde{b}_1 - b_1) + c_2((\tilde{a}_2 - a_2)x_j + \tilde{b}_2 - b_2) 
    \end{align*}
    Now, substituting the changes computed in equation \eqref{eq:case12} and equation \eqref{eq:case13}, we get that
     \begin{equation*}
     \tilde{f}(x_j) - f(x_j) = c_1a_1\epsilon(x_0 - x_j) + c_2\left(-\frac{c_1}{c_2}\epsilon a_1(x_0 - x_j)\right) = 0
    \end{equation*}
    Hence, we see that the function values are preserved for datapoints in this range. This is because the changes in the parameters $a_2$ and $b_2$ were chosen in such a way so that the change in function value introduced due to the change in parameters of unit 1 can be balanced.
    Calculating the change in regularization value with the perturbed parameters, we get that the regularization term strictly decreases since $0 < \tilde{a}_1x_j+\tilde{b}_1 < a_1x_j + b_1$ by \eqref{eq:case12} which we have already argued before. 
       \begin{equation*}
     \tilde{R}_{1j} - R_{1j} = (\sigma(\tilde{a}_1x_j + \tilde{b}_1))^2 + c_1^2(1 + x_j^2)I_{\tilde{a}_1x_j + \tilde{b}_1 > 0} -     (\sigma({a}_1x_j + {b}_1))^2 - c_1^2(1 + x_j^2)I_{{a}_1x_j + {b}_1 > 0} < -\Theta(\eps)
     \end{equation*}
     Similarly, since $0 \leq \tilde{a}_2x_j+\tilde{b}_2 < a_2x_j + b_2$ by equation \eqref{eq:case13}, the regularization value for unit 2 strictly decreases for this range of datapoints.
           \begin{equation*}
     \tilde{R}_{2j} - R_{2j} = (\sigma(\tilde{a}_2x_j + \tilde{b}_2))^2 + c_2^2(1 + x_j^2)I_{\tilde{a}_2x_j + \tilde{b}_2 > 0} - (\sigma({a}_2x_j + {b}_2))^2 - c_2^2(1 + x_j^2)I_{{a}_2x_j + {b}_2 > 0}< -\Theta(\eps).
     \end{equation*}
    \paragraph{Case 2 (\boldmath $a_1 > 0, c_1 < 0, a_2 < 0, c_2 > 0):$\unboldmath}
    This case corresponds to the third pane in Figure~\ref{fig:allcases}.
\begin{equation*}
\begin{gathered}
\tilde{a}_1 = a_1(1 - \epsilon) \qquad \tilde{b}_1 = b_1 + x_0a_1\epsilon\\
\tilde{a}_2= a_2  +\frac{c_1}{c_2}(\tilde{a}_1 - a_1) \qquad \tilde{b}_2= b_2  +\frac{c_1}{c_2}(\tilde{b}_1 - b_1)\\
a = -c_1(\tilde{a}_1 - a_1) \qquad b = -c_1(\tilde{b}_1 - b_1)\\
\end{gathered}
\end{equation*}
Similarly to the previous case, we can argue that the function value at the datapoints remain same and regularization decreases by $\Theta(\eps).$
    \paragraph{Case 3 \boldmath $(a_1 < 0, c_1 < 0, a_2 > 0, c_2 > 0):$ \unboldmath}
    This case corresponds to the fourth pane in Figure~\ref{fig:allcases}:
            \begin{equation*}
\begin{gathered}
\tilde{a}_1 = a_1(1 - \epsilon) \qquad \tilde{b}_1 = b_1 + x_0a_1\epsilon\\
\tilde{a}_2= a_2  +\frac{c_1}{c_2}(\tilde{a}_1 - a_1) \quad \tilde{b}_2= b_2  +\frac{c_1}{c_2}(\tilde{b}_1 - b_1)\\
a = -c_1(\tilde{a}_1 - a_1) \qquad b = -c_1(\tilde{b}_1 - b_1)\\
\end{gathered}
\end{equation*}
Similarly to the previous case, we can argue that the function value at the datapoints remain same and regularization decreases by $\Theta(\eps).$
    \paragraph{Case 4 \boldmath $(a_1 < 0, c_1 < 0, a_2 < 0, c_2 > 0):$ \unboldmath}
    This case corresponds to the right pane in Figure~\ref{fig:allcases}:
        \begin{equation*}
\begin{gathered}
\tilde{a}_1= a_1  -\frac{c_2}{c_1}(\tilde{a}_2 - a_2) \qquad \tilde{b}_1= b_1  -\frac{c_2}{c_1}(\tilde{b}_2 - b_2)\\
\tilde{a}_2 = a_2(1 - \epsilon) \qquad \tilde{b}_2 = b_2 + x_0a_2\epsilon
\end{gathered}
\end{equation*}
    Similarly to the previous case, we can argue that the function value at the datapoints remains the same and regularization  decreases by $\Theta(\eps).$
%
     %
%
%
\end{proof}

\section{Tanh and Logistic Activations (Proof of Theorem~\ref{thm:n=1})}\label{ap:n=1}
Here, we  discuss the implications of our characterization of stable points in the dynamics of SGD with label noise, for networks with either hyperbolic tangent activations or logistic activations. In particular, we will consider networks with two layers, of arbitrary width, that are trained on a single $d$-dimensional data point $(x,y)$. We find that, at ``non-repellent'' points, the neurons can be partitioned into a constant number of essentially equivalent neurons, and thus the network provably emulates a constant-width network on ``simple" data.  

Throughout this section we denote our single training point by $(x,y)$, where $x \in \mathbb{R}^d$ and $y \in \mathbb{R}$, and we assume $x \neq 0$. Our network is a two layer network, parameterized by a length $n$ vector $c$ and a $d\times n$ matrix $w$, and represents the function
\begin{equation*}
    f(x; c, w) = \sum_{i=1}^n c_i \sigma (w_i^\intercal x)
\end{equation*}
where $c \in \mathbb{R}^n$ and $w_1,...,w_n$ are the columns of $w \in \mathbb{R}^{d \times n}$. In Section~\ref{sec:logistic} below, the activation function $\sigma$ will be the logistic function, while in Section~\ref{sec:tanh} we analyze the tanh activation function.  Since we are only concerned with the network's behavior on a single data point $(x,y)$, unlike in the body of the paper where the subscript $i$ typically denoted a choice of data point, here we use the subscript $i$ to index the hidden units of the network. We  let $h_i = \sigma(w_i^t x)$ denote the value of the $i^{\text{th}}$ hidden unit and let $o_i = c_i h_i$ denote the output (after scaling) of the $i^{\text{th}}$ hidden unit. Then, we simply have that $f(x; c,h) = \sum_{i=1}^n o_i$.

\subsection{``Non-repellent'' points for logistic activation}\label{sec:logistic}
We prove the following proposition, establishing the portion of Theorem~\ref{thm:n=1} concerning logistic activation functions:
\begin{proposition}
\label{prop:logistic stable}
    Let $\theta = (c,w)$ parameterize a two-layer network with logistic activations. If $\theta$ is ``non-repellent'' according to Definition~\ref{def:nonrepellent} for the dynamics of training with a single d-dimensional datapoint $(x,y)$ where $x \neq 0$, then there exists $\alpha_1,\alpha_2$ and $\beta_1, \beta_2$ such that for each hidden unit $i$, either $c_i = \alpha_1$ and $h_i = \beta_1$ or $c_i = \alpha_2$ and $h_2 = \beta_2$.
\end{proposition}
\begin{proof}
First, we derive the implicit regularizer, $R$, for a two layer network with logistic activations.  We compute:

\begin{equation*}
{\nabla_{c_i}f(x;c,w)} = h_i \qquad
{\nabla_{w_{ij}}f(x;c,w)} = c_ih_i(1-h_i)x_j 
\end{equation*}

Thus, 
\begin{equation*}
R = {||\nabla_{w,c}f(x;c,w)||^2} = \sum_i \big[h_i^2 + c_i^2 h_i^2(1 - h_i)^2||x||^2\big]
\end{equation*}
Recall that a choice of parameters with zero-error is ``non-repellant'' iff the implicit regularizer has zero gradient in the span of directions with zero function gradient. Thus, we want to consider directions that do not change the error, up to first order. Recall that we defined $o_i = c_i h_i$ and that the networks output is just $\sum_{i=1}^n o_i$. Any change to the parameters that leaves all the $o_i$ the same must leave the network output the same, and thus the error unchanged as well. First, we investigate for what choices of parameters do there not exist any directions that leave all $o_i$ constant but decrease the regularization term. We rewrite the regularization term using $o_i$:
\begin{equation}
\label{eq:logistic R from h}
    R = {||\nabla_{w,c}f(x;c,w)||^2} = \sum_i \big[h_i^2 + o_i^2(1 - h_i)^2||x||^2\big]
\end{equation}
Suppose for some $i$ that the derivative of the above expression with respect to $h_i$ is nonzero. Then, we can change $w_i$ in the direction that slightly increases $h_i$ while also decreasing $c_i$ just enough to keep $o_i$ constant. That direction would keep the error at $0$ but the implicit regularization term would have nonzero directional derivative in it. Thus, for ``non-repellent'' points, we must have that the following is $0$ for all $i$:
\begin{equation*}
   \frac{\partial}{\partial h_i}R = 2h_i + 2(h_i - 1)o_i^2 ||x||^2 = 0
\end{equation*}
We solve the above equation for $h_i$ to determine that at all ``non-repellent'' points:
\begin{equation}
\label{eq:logistic optimal h}
    h_i = \frac{o_i^2||x||^2}{1 + o_i^2 || x||^2}
\end{equation}
We can plug this back into equation \ref{eq:logistic R from h} to determine that at ``non-repellent'' points the following must be true:
\begin{equation*}
    R = \sum_i \big[(\frac{o_i^2||x||^2}{1 + o_i^2 || x||^2})^2 + o_i^2(1 - \frac{o_i^2||x||^2}{1 + o_i^2 || x||^2})^2||x||^2\big] = \sum_i  \frac{o_i^2||x||^2}{1 + o_i^2 || x||^2}
\end{equation*}
For convenience, we define $R_o(z) = \frac{z^2||x||^2}{1 + z^2 || x||^2}$. Then, we have that at ``non-repellent'' points, $R = \sum_{i=1}^n R_o(o_i)$. The function $R_o$, as well as its derivative, is depicted in Figure \ref{fig:R plot logistic}.

Next, we consider the effect of changing two units at a time. We claim that if there are units $i,j$ where $R_o'(o_i) \neq R_o'(o_j)$, then we are not at a ``non-repellent'' point. Consider moving the parameters in a direction that increases $o_i$ by $\epsilon$ and decreases $o_j$ by $\epsilon$. That direction will leave the network output constant, and therefore also the error. Furthermore, we can choose the direction so that it additionally modifies $h_i$ and $h_j$ so that they satisfy equation \ref{eq:logistic optimal h} with respect to the modified $o_i$ and $o_j$. Altogether, this means that $R$ changes by $(R_o(o_i + \epsilon) - R_o(o_i)) - (R_o(o_j + \epsilon) - R_o(o_j)) $. The result is that, after a change by $\epsilon$, the new regularization penalty will change (up to first-order approximation) by $\epsilon(R_o'(o_i) - R_o'(o_j))$, which is nonzero. Thus, $R$ decreases linearly in the direction we constructed, implying we are not at a ``non-repellent'' point, yielding the desired contradiction.

\begin{figure}[tb]
  \centering
  \includegraphics[width=0.4\textwidth]{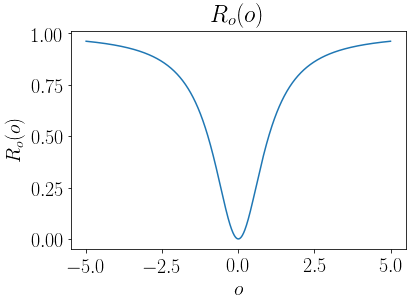}\quad \includegraphics[width=0.4\textwidth]{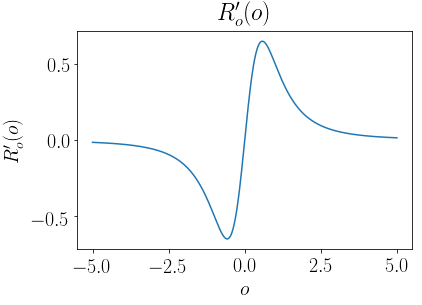}
\caption{\small{Plots depicting the function $R_o$ on the left and its derivative on the right, for $||x|| = 1$. From the plots, we see that the equation $R_o'(o) = a$ has at most two solutions for any choice of $a$. Other choices of $||x||$ would only stretch the plots, which does not affect that conclusion}\label{fig:R plot logistic}}
\end{figure}

Thus, at a ``non-repellent'' point we must have that $R_o'(o_i)$ is the same for all $o_i$. Thus the number of different values of $o_i$ is upper bounded by the number of solutions to the equation $R_o'(o) = a$ where $a$ is some scalar. See Figure \ref{fig:R plot logistic} for a plot illustrating that this equation has at most 2 solutions. To prove this,  we first compute the derivative and set it equal to $a$
\begin{equation*}
    R_o'(o) = \frac{2 o ||x||^2}{(1 + o^2||x||^2)^2} = a \implies
    a(1 + o^2||x||^2)^2 - 2o||x||^2 = 0
\end{equation*}
Since $||x||\neq 0$, the function $a(1 + o^2||x||^2)^2 - 2o||x||^2$ is a strictly convex function of $o$ for $a>0$, is strictly concave for $a<0$, and a linear function when $a=0$, and thus in all cases has at most $2$ solutions for $||x||\neq 0$. Thus, at a ``non-repellent'' point, there are at most two distinct values for $o_1,...,o_n$. Furthermore, we have already shown that at ``non-repellent'' points, $h_i$ is a function of $o_i$. It also follows that $c_i = o_i/h_i$ is a function of $o_i$. Thus, if $o_i = o_j$ then $c_i = c_j$ and $h_i = h_j$, so all units with the same output ($o_i$) also share the same value for $c_i$ and $h_i$. Hence, there are at most two possible values for $c_i,h_i$, which we can name $\alpha_1,\beta_1$ and $\alpha_2, \beta_2$, proving this proposition.
\end{proof}

\subsection{``Non-repellent'' points for tanh activation}\label{sec:tanh}
The following proposition characterizes the portion of Theorem~\ref{thm:n=1} concerning tanh activations.

\begin{proposition}
\label{prop:tanh stable}
    Let $\theta = (c,w)$ parameterize a two-layer network with tanh activations. If $\theta$ is ``non-repellent'' according to Definition~\ref{def:nonrepellent} for the dynamics of training with a single d-dimensional datapoint $(x,y)$ where $x \neq 0$, then there exists $\alpha$ and $\beta$ such that for each hidden unit $i$, either $c_i = \alpha$ and $h_i = \beta$ or $c_i = -\alpha$ and $h_2 = -\beta$ or $c_i = h_i = 0$.
\end{proposition}

The proof of this proposition is mostly the same as the proof of proposition \ref{prop:logistic stable}. However, instead of proving that every point in the range of $R_o'(o)$ is attained by at most two points in the domain, we will prove that the function is injective. The other difference is that $R_o'(o)$ is undefined at $o = 0$, so in addition to the units that are $h_i$ and $c_i$ (up to sign), there can also be units with $0$ output.  Due to the highly repetitive logic, we go through this proof at a faster pace than \ref{prop:logistic stable}. 
\begin{proof}
For a two layer network with tanh activations, the implicit regularizer is
\begin{equation}
\label{eq:tanh R from h}
R  = \sum_i \big[h_i^2 + c_i^2 (1 - h_i^2)^2||x||^2\big] = \sum_i \big[h_i^2 + \frac{o_i^2}{h_i^2} (1 - h_i^2)^2||x||^2\big] 
\end{equation}

At ``non-repellent'' points, we must have that the below derivative is $0$ for all $i$
\begin{equation*}
   \frac{\partial}{\partial h_i}R = \frac{2 h_i^{4} + 2 h_i^{4} o_i^{2} ||x||^{2}  - 2 o_i^{2} ||x||^{2}}{h_i^3} = 0
\end{equation*}
We solve the above equation for $h_i^2$ to determine that at all ``non-repellent'' points:
\begin{equation}
\label{eq:tanh optimal h}
    h_i^2 = \sqrt{\frac{o_i^{2} ||x||^{2}}{o_i^{2} ||x||^{2} + 1}}
\end{equation}

\begin{figure}[tb]
  \centering
  \includegraphics[width=0.4\textwidth]{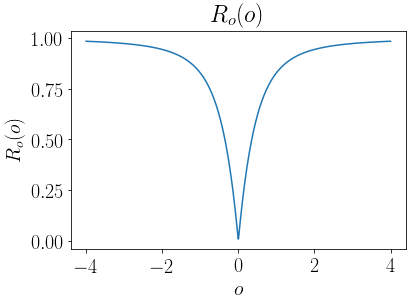}\quad \includegraphics[width=0.4\textwidth]{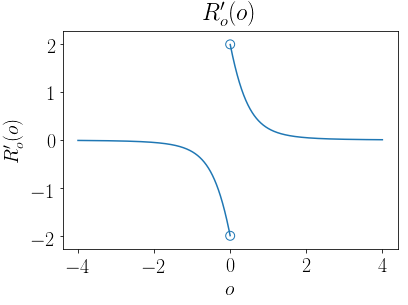}
\caption{\small{Plots depicting the function $R_o$ on the left and its derivative on the right, for $||x|| = 1$. From the plots, we see that $R_o'(o)$ is injective and undefined at $o=0$. Other choices of $||x||$ would only stretch the plots, which does not affect that conclusion.}\label{fig:R plot tanh}}
 
\end{figure}

We  plug this back into equation \ref{eq:tanh R from h} and simplify to determine that, at ``non-repellent'' points, the following must be true:
\begin{align*}
    R &= \sum_i \big[ h_i^2 + \frac{o_i^2}{h_i^2}(1-h_i^2)^2||x||^2 \big ]\\
    &=  \sum_i\frac{h_i^4(1 + o_i^2||x||^2) + o_i^2||x||^2(1-2h_i^2) }{h_i^2}\\
    & =  \sum_i \frac{ o_i^2||x||^2 + o_i^2||x||^2(1-2h_i^2) }{h_i^2}\\
    & = \sum_i2o_i^2||x||^2(\frac{1 }{h_i^2}-1)\\
    & = \sum_i2o_i^2||x||^2(\sqrt{\frac{o_i^{2} ||x||^{2} + 1}{o_i^{2} ||x||^{2}}}-1)\\
    &= \sum_i2 \big[(\sqrt{o_i^2||x|^2(o_i^{2} ||x||^{2} + 1)}-o_i^2||x||^2)\big]
\end{align*}
We define $R_o(o_i) = 2 \big[(\sqrt{o_i^2||x||^2(o_i^{2} ||x||^{2} + 1)}-o_i^2||x||^2)\big]$. Recall that we showed in the proof of Proposition~\ref{prop:logistic stable} that if there exists two units, $i,j$, such that $R_o'(o_i) \neq R_o'(o_j)$, then we cannot be at a ``non-repellent'' point. In this case, it turns out that $R_o'(o)$ is undefined at $o=0$, which means any number of units can have zero output. However, at all other points, $R_o'(o)$ is injective. This means that all units that don't have $0$ output must share the same output. See Figure~\ref{fig:R plot tanh} for illustrative plots.

To show that $R_o'$ is injective, we first take $||x||=1$ without loss of generality, since the argument of $R_o$ always appears multiplied by $||x||$. Next, we differentiate and simplify to obtain \[R_o'(z)=2z\frac{z^2+\frac{1}{2}-\sqrt{z^4+z^2}}{\sqrt{z^4+z^2}},\] which is easily seen to have the same sign as $z$ (and is undefined when $z=0$). Further, the 2nd derivative---ignoring its value at 0---simplifies to the following expression: \[R_o''(z)=2|z|\frac{z^2+\frac{3}{2}}{(z^2+1)^{3/2}}-2,\] which is seen to be negative everywhere. Thus for positive $z$, we have $R_o'(z)$ is positive and decreasing, while for negative $z$ it is negative and decreasing, implying $R_o'(z)$ is injective, as desired.

We thus conclude that at ``non-repellent'' points, all units have either the same output ($o_i$) or have output $0$. From Equation \ref{eq:tanh optimal h}, we know that at ``non-repellent'' points $h_i^2$ is a function of $o_i$. Furthermore, $c_i^2 = o_i^2 / h_i^2$, so $c_i^2$ is also a function of $o_i$. Thus, all units that don't have output $0$ must have the same $h_i$ and $c_i$ (up to sign), and since they have the same $o_i = h_i c_i$, the signs must match up as well. This means that, at ``non-repellent'' points, there is $\alpha,\beta$ so that for each hidden unit $i$ where $o_i \neq 0$, either $c_i = \alpha$ and $h_i = \beta$ or $c_i = -\alpha$ and $h_i = -\beta$.

Finally, we show that at ``non-repellent'' points, if $o_i = 0$ then $h_i = c_i = 0$.  This means that not only does the $i^{\text{th}}$ unit not affect the output of the network at this particular choice of $x$, but it also does not affect the output of the network for any input. If $o_i = 0$ then from equation \ref{eq:tanh optimal h} we know $h_i = 0$. Recall that the networks output is $\sum_{i=1}^n c_i h_i$, so if $h_i = 0$, then changing $c_i$ does not affect the error. Therefore, we could only be at a ``non-repellent'' point if $\frac{\partial R}{\partial c_i} = 0$. Taking the derivative of equation \ref{eq:tanh R from h} we see $\frac{\partial R}{\partial c_i} = 2c_i(1-h_i^2)^2 ||x||^2$ which is zero only if $c_i = 0$. Thus, if $o_i = 0$ then $h_i = c_i = 0$.

\end{proof}

\end{document}